\documentclass[onefignum,onetabnum]{siamart190516}

\ifpdf
  \DeclareGraphicsExtensions{.eps,.pdf,.png,.jpg}
\else
  \DeclareGraphicsExtensions{.eps}
\fi

\newsiamremark{remark}{Remark}
\newsiamremark{example}{Example}
\newsiamremark{hypothesis}{Hypothesis}
\crefname{hypothesis}{Hypothesis}{Hypotheses}
\newsiamthm{claim}{Claim}
\newsiamthm{assumption}{Assumption}
\crefname{assumption}{Assumption}{Assumptions}

\headers{On the Omnipresence of Spurious Local Minima}
{Constantin Christof and Julia Kowalczyk}

\title{On the Omnipresence of Spurious Local Minima in Certain 
Neural Network Training Problems}

\author{Constantin Christof\thanks{Technical University of Munich,
Chair of Optimal Control,
Center for Mathematical Sciences, 
Boltzmannstraße 3,
85748 Garching, Germany, \email{christof@ma.tum.de}, \email{julia.kowalczyk@ma.tum.de}
}
\and
Julia Kowalczyk\footnotemark[1]
}

\usepackage{amsopn}
\usepackage{amsfonts}
\usepackage{graphicx}
\usepackage{amssymb}
\usepackage{enumitem}
\usepackage{dsfont}
\usepackage{subcaption}

\newcommand{\sgn}{\operatorname{sgn}}
\newcommand{\closure}{\operatorname{cl}}
\newcommand{\dd}{\hspace{0.03cm}\mathrm{d}}
\renewcommand{\span}{\operatorname{span}}
\newcommand{\diam}{\operatorname{diam}}

\newcommand{\BB}{\mathcal{B}}

\newcommand{\FF}{\mathcal{F}}

\newcommand{\HH}{\mathcal{H}}

\newcommand{\LL}{\mathcal{L}}
\newcommand{\MM}{\mathcal{M}}

\newcommand{\R}{\mathbb{R}}

\DeclareMathOperator*{\argmin}{arg\,min}

\begin{document}

\maketitle

\begin{abstract}
We study the loss landscape of training problems 
for deep artificial neural networks with a one-dimensional real output whose activation functions 
contain an affine segment and whose hidden layers have width
at least two. It is shown that such problems 
possess a continuum of spurious (i.e., not globally optimal) 
local minima for all target functions that are not affine.
In contrast to previous works, our analysis covers all sampling and parameterization regimes,
general differentiable loss functions, arbitrary continuous nonpolynomial activation functions, 
and both the finite- and infinite-dimensional setting.
It is further shown that the appearance of the spurious local minima in the considered training problems is a 
direct consequence of  the universal approximation theorem and that the underlying mechanisms
also cause, e.g., $L^p$-best approximation problems to be ill-posed in the sense of Hadamard
for all networks that do not have a dense image.
The latter result also holds without the assumption of local affine linearity
and without any conditions on the hidden layers. 
\end{abstract}



\section{Introduction}
\label{sec:1}

Due to its importance for the understanding of 
the behavior, performance, and limitations of machine learning algorithms, 
the study of the loss landscape of training problems for artificial neural networks 
has received considerable attention in the last years.
Compare, for instance, with the early works \cite{Auer1996,Blum1992,Sima2002} on this topic, 
with the contributions on stationary points and plateau
phenomena in \cite{Ainsworth2021,Cheridito2021-2,Cooper2020,Dauphin2014,Yoshida2019}, 
with the results on suboptimal local minima and valleys in \cite{Arjevani2019SymII,Christof2021,Ding2020,He2020,Nguyen2018,Petzka2021,Venturi2019,Yun2018},
and with the overview articles \cite{Berner2021,Sun2019,Sun2020}. 
For fully connected feedforward neural networks
involving activation functions with an affine segment, 
much of the research on landscape properties was initially 
motivated by the observation of Kawaguchi \cite{Kawaguchi2016} that 
networks with linear activation functions give rise to learning problems 
that do not possess spurious (i.e., not globally optimal) local minima and thus 
behave -- at least as far as the notion of local optimality is concerned --
like convex problems. 
For related work on this topic
and generalizations of the results of \cite{Kawaguchi2016},
see also \cite{Eftekhari2020,Laurent2018,Saxe2013,Yun2018,Zou2020}. 
Based on the findings of \cite{Kawaguchi2016},
it was conjectured that ``nice'' landscape properties or even the 
complete absence of spurious local minima can also be established 
for nonlinear activation functions in many situations and that this behavior is one of the 
main reasons for the performance that machine learning algorithms achieve in 
practice, cf.\ \cite{Eftekhari2020,Saxe2013,Yu1995}. It was quickly realized, however, that,
in the nonlinear case, the situation is more complicated
and that examples of training problems with spurious local minima 
can readily be constructed even when only ``mild'' nonlinearities are present 
or the activation functions are piecewise affine. 
Data sets illustrating this 
for certain activation functions can be found, for example, in \cite{Safran2018,Swirszcz2016,Yun2018}.
On the analytical level, one of the first 
general negative results on the landscape properties of training problems 
for neural networks 
was proven by Yun et al.\ in \cite[Theorem 1]{Yun2018}. 
They showed that spurious local minima 
are indeed \emph{always} present when
a finite-dimensional squared loss training problem for a 
one-hidden-layer neural network with a one-dimensional real output, a hidden layer of width at least two,
and a leaky ReLU-type activation function
is considered and the training data cannot be precisely fit with an affine function. 
This existence result was later also generalized in 
\cite[Theorem 1]{He2020} and \cite[Theorem 1]{Liu2021} to finite-dimensional training problems 
with arbitrary loss for deep networks with piecewise affine activation functions, 
in \cite[Corollary~1]{Ding2020} to finite-dimensional squared loss problems
for deep networks with locally affine activations under the assumption of realizability, 
and in \cite[Corollary 47]{Christof2021} to finite-dimensional squared loss problems 
for deep networks involving many commonly used activation functions. 
For contributions on spurious minima in the absence of local affine linearity,  
see \cite{Christof2021,Ding2020,Petzka2021,Swirszcz2016,Yun2018}.

The purpose of the present paper is to prove
that the results of \cite{Yun2018} on the existence of spurious local minima 
in training problems for neural networks 
with piecewise affine activation functions are also true in a far more general setting
and that the various assumptions on the activations,
the loss function, 
the network architecture, and the realizability of the data in \cite{Christof2021,Ding2020,He2020,Yun2018}
can be significantly relaxed.  
More precisely, we show that \cite[Theorem 1]{Yun2018}
can be straightforwardly extended to networks of arbitrary depth,
to arbitrary continuous nonpolynomial activation functions with an affine segment,
to all (sensible) loss functions, and to infinite dimensions. 
We moreover establish that there is 
a whole continuum of spurious local minima in the situation of
\cite[Theorem 1]{Yun2018} whose Hausdorff dimension can be estimated from below. 
For the main results of our analysis, we refer the reader to \cref{th:spuriousminima-nonconstant,th:spuriousminima-constant}. 
Note that these theorems in particular imply that the observations made in \cite{He2020,Liu2021,Yun2018} 
are not a consequence of the piecewise affine linearity of the activation functions considered in
these papers but of general effects
that apply to all nonpolynomial continuous activation functions 
with an affine segment (SQNL, PLU, ReLU, leaky/parametric ReLU, ISRLU, ELU, etc.),
that network training without spurious local minima 
is impossible (except for the pathological situation of affine linear training data)
when the simple affine structure of \cite{Kawaguchi2016} is kept locally 
but a global nonlinearity is introduced to enhance the approximation capabilities of the network, 
and that there always exist choices of hyperparameters 
such that gradient-based solution algorithms terminate with a suboptimal point
when applied to training problems of the considered type. 

We would like to point out that establishing the 
existence of local minima in training problems for neural networks 
whose activation functions possess an affine segment is
not the main difficulty in the context of 
\cref{th:spuriousminima-nonconstant,th:spuriousminima-constant}. 
To see that such minima are present, it suffices to exploit 
that neural networks with locally affine activations 
can emulate linear neural networks, see \cref{lem:aff1,lem:aff2},
and this construction has already been used in various papers on the landscape properties 
of training problems, e.g., \cite{Christof2021,Ding2020,Goldblum2020Truth,He2020,Yun2018}.
What is typically considered difficult in the literature is proving that 
the local minima obtained from the affine linear segments of the activation functions are 
indeed \emph{always} spurious -- independently of the precise 
form of the activations, the loss function,
the training data, and the network architecture. 
Compare, for instance, with the comments in 
\cite[section~2.2]{Yun2018},
\cite[section~2.1]{Christof2021}, and 
\cite[section~3.3]{He2020} on this topic.
In existing works on the loss surface of neural networks, the 
problem of rigorously proving the spuriousness of local minima 
is usually addressed by manually constructing network parameters 
that yield smaller values of the loss function, cf.\ the proofs of \cite[Theorem~1]{Yun2018}, 
\cite[Theorem 1]{Liu2021}, and 
\cite[Theorem~1]{He2020}. 
Such constructions ``by hand'' are naturally only possible when simple activation functions 
and network architectures 
are considered and not suitable to 
obtain general results. One of the main points that we would like to communicate with this paper is that the
spuriousness of the local minima in 
\cite[Theorem 1]{Yun2018}, 
\cite[Theorem 1]{He2020},
\cite[Corollary~47]{Christof2021},
\cite[Theorem 1]{Liu2021},
and 
\cite[Corollary~1]{Ding2020}
and also our more general \cref{th:spuriousminima-nonconstant,th:spuriousminima-constant}
is, in fact, a straightforward consequence of the universal approximation theorem
in the arbitrary width formulation as proven by Cybenko, Hornik, and Pinkus in \cite{Cybenko1989,Hornik1991,Pinkus1999}, 
or, more precisely, the fact that the universal approximation theorem implies that 
the image of a neural network with a fixed architecture does not possess 
any supporting half-spaces in function space; see \cref{th:nohalfspaces}.
By exploiting this observation, we can easily overcome 
the assumption of \cite{He2020,Liu2021,Yun2018} that the activation functions are piecewise affine linear,
the restriction to the one-hidden-layer case in \cite{Yun2018}, 
the restriction to the squared loss function in \cite{Christof2021,Ding2020,Yun2018}, 
and the assumption of realizability in \cite{Ding2020} and are moreover able to extend the 
results of these papers to infinite dimensions. 

Due to their connection to the universal approximation theorem,
the proofs of \cref{th:spuriousminima-nonconstant,th:spuriousminima-constant} also highlight
the direct relationship that exists between the approximation capabilities 
of neural networks and the optimization landscape and well-posedness properties of 
the training problems that have to be solved in order to determine a
neural network best approximation. 
For further results on this topic, we refer to \cite{Cohen2021} and \cite[section 6]{Pinkus1999}, 
where it is discussed that every approximation instrument that asymptotically 
achieves a certain rate of convergence for the approximation error in terms of its number of 
degrees of freedom
necessarily gives rise to numerical algorithms that are unstable.  
In a spirit similar to that of \cite{Cohen2021}, we show in \cref{sec:5} that the nonexistence of 
supporting half-spaces exploited in the proofs of 
\cref{th:spuriousminima-nonconstant,th:spuriousminima-constant}
also  immediately implies that 
best approximation problems for neural networks 
posed in strictly convex Banach spaces with strictly convex duals 
are always ill-posed in the sense of Hadamard when the considered network 
does not have a dense image. 
Note that this result holds regardless of whether the activation functions
possess an affine segment or not and without any assumptions on the widths of the hidden 
layers. We remark that, for one-hidden-layer networks, 
the corollaries in \cref{sec:5} have essentially already been proven in \cite{Kainen1999,Kainen2001},
see also \cite{Petersen2021}.
Our analysis extends the considerations of \cite{Kainen1999,Kainen2001} to arbitrary depths. 

We conclude this introduction with an overview of the content and the structure of the 
remainder of the paper:

\Cref{sec:2} is concerned with preliminaries. 
Here, we introduce the notation, the functional analytic setting, and the
standing assumptions that we use in this work.  
In \cref{sec:3}, we present our main results on the existence 
of spurious local minima, see \cref{th:spuriousminima-nonconstant,th:spuriousminima-constant}. This section also 
discusses the scope and possible extensions of our analysis and demonstrates that 
\cref{th:spuriousminima-nonconstant,th:spuriousminima-constant}
cover the squared loss problem studied in \cite[Theorem 1]{Yun2018} 
as a special case.
\Cref{sec:4} contains the proofs of \cref{th:spuriousminima-nonconstant,th:spuriousminima-constant}.
In this section, we establish that the 
universal approximation theorem
indeed implies that the image of a neural network in function space 
does not possess any supporting half-spaces and show that this property
allows us to prove the spuriousness of local minima in a natural way.
In \cref{sec:5}, we discuss further implications of the geometric properties 
of the images of neural networks exploited in \cref{sec:4}. 
This section contains the already mentioned results on the Hadamard ill-posedness
of neural network best approximation problems posed in
strictly convex Banach spaces with strictly convex duals. 
Note that tangible examples of such spaces are $L^p$-spaces with 
$1 < p < \infty$, see \cref{cor:LPillposed}. The paper concludes with 
additional comments on
the results derived in \cref{sec:3,sec:4,sec:5} and remarks on open problems.

\section{Notation, preliminaries, and basic assumptions}
\label{sec:2}
Throughout this work, $K \subset \R^d$, $d \in \mathbb{N}$, denotes 
a nonempty compact subset of the Euclidean space $\R^d$.
We endow $K$ with the subspace topology $\tau_K$ induced by the standard topology on $(\R^d, |\cdot|)$,
where \mbox{$|\cdot|$} denotes the Euclidean norm, and denote the associated 
Borel sigma-algebra on $K$ with $\BB(K)$. 
The space of  continuous functions $v\colon K \to \R$ equipped with the maximum norm 
$\|v\|_{C(K)} := \max \{|v(x)| \colon x \in K\}$
is denoted  by $C(K)$. As usual, we identify  the topological
dual space $C(K)^*$ of $(C(K), \|\cdot \|_{C(K)})$ with the space 
$\MM(K)$ of signed Radon measures on $(K, \BB(K))$ endowed with the 
total variation norm $\| \cdot \|_{\MM(K)}$, see 
\cite[Corollary 7.18]{Folland1999}. The corresponding dual pairing is denoted by 
\mbox{$\langle \cdot, \cdot \rangle_{C(K)}\colon \MM(K) \times C(K) \to \R$.}
For the closed cone of nonnegative measures in $\MM(K)$,
we use the notation $\MM_+(K)$.  
The standard, real Lebesgue spaces associated with a measure space
$(K, \BB(K), \mu)$, $\mu \in \MM_+(K)$,
are denoted by $L^p_\mu(K)$, $1 \leq p \leq \infty$, and equipped with the usual norms $\smash{\|\cdot\|_{L_\mu^p(K)}}$, 
see \cite[section 5.5]{Benedetto2010}. For the open ball of radius $r > 0$ 
in a normed space $(Z, \|\cdot\|_Z)$ centered at a point $z \in Z$, we use the symbol $B_r^Z(z)$,
and for the topological closure of a set $E \subset Z$, the symbol $\closure_Z(E)$.

The neural networks that we study in this paper are standard fully connected feedforward neural networks
with a $d$-dimensional real input and a one-dimensional real output
(with $d$ being the dimension of the Euclidean space $\R^d \supset K$). 
We denote the number of hidden layers of a network with $L \in \mathbb{N}$
and the widths of the hidden layers with $w_i \in \mathbb{N}$, $i=1,...,L$. 
For the ease of notation, we also introduce the definitions $w_0 := d$ and $w_{L+1} := 1$
for the in- and output layer. The weights and biases are denoted by $A_{i} \in \R^{w_{i} \times w_{i-1}}$
and $b_{i} \in \R^{w_{i}}$, $i=1,...,L+1$, respectively,
and the activation functions of the layers by $\sigma_i\colon \R \to \R$, $i=1,...,L$.
Here and in what follows, all vectors of real numbers 
are considered as column vectors. 
We will always assume that the functions $\sigma_i$ are continuous, i.e., $\sigma_i \in C(\R)$ for
all $i=1,...,L$.
To describe the action of the network layers, we 
define $\varphi_i^{A_i, b_i}\colon \R^{w_{i-1}} \to \R^{w_i}$, $i=1,...,L+1$, to be the functions 
\[
\varphi_i^{A_i, b_i}(z) := \sigma_i\left (A_i z + b_i \right ),~\forall i=1,...,L,
\qquad \varphi_{L+1}^{A_{L+1}, b_{L+1}}(z) := A_{L+1}z + b_{L+1},
\]
with $\sigma_i$ acting componentwise on the entries of $ A_{i}z + b_{i} \in \R^{w_i}$. 
Overall, this notation allows us to denote a feedforward neural network in the following way:
\begin{equation}
\label{eq:network}
\psi(\alpha, \cdot)\colon \R^d \to \R,\qquad 
\psi(\alpha, x) 
:= \left ( \varphi_{L+1}^{A_{L+1}, b_{L+1}} \circ ... \circ \varphi_{1}^{A_{1}, b_{1}} \right )(x).
\end{equation}
Here, we have introduced the variable
$\alpha := 
\{ (A_{i}, b_{i})\}_{i=1}^{L+1}
$
as an abbreviation for the collection of all network parameters
and the symbol ``$\circ$'' to denote a composition.
For the set of all possible $\alpha$, i.e., the parameter space of a network, we write
\[
D := 
\left \{
\alpha = 
\{ (A_{i}, b_{i})\}_{i=1}^{L+1}
\; \Big | \;
A_{i} \in \R^{w_{i} \times w_{i-1}},\,
b_{i} \in \R^{w_{i}},~
\forall i=1,...,L+1
\right \}.
\]
We equip the parameter space $D$ with the Euclidean norm $|\cdot|$ of the space $\R^{m}$,
$m := w_{L+1} (w_{L} + 1) + ... + w_{1}(w_{0} + 1)$, 
that $D$ can be transformed into by rearranging the entries of $\alpha$. 
Note that this implies that $m = \dim(D)$ holds, where $\dim(\cdot)$ denotes the 
dimension of a vector space in the sense of linear algebra.
Due to the continuity of the activation functions $\sigma_i$, 
the map $\psi\colon D \times \R^d \to \R$ in \eqref{eq:network} 
gives rise to an operator from $D$ into the space $C(K)$. 
We denote this operator by $\Psi$, i.e., 
\begin{equation}
\label{eq:Psidef}
\Psi\colon D \to C(K),\qquad \Psi(\alpha) := \psi(\alpha,\cdot)\colon K \to \R.
\end{equation}
Using the function $\Psi$, we can formulate the training problems that we are interested in 
as follows:
\begin{equation*}
\label{eq:P}
\tag{P}
		\text{Minimize} \quad  \LL(\Psi(\alpha), y_T)\qquad \text{w.r.t.}\quad \alpha \in D.
\end{equation*}
Here, $\LL\colon C(K) \times C(K)\to \R$ denotes the loss function and $y_T \in C(K)$ the target function. 
We call $\mathcal{L}$ Gâteaux differentiable in
its first argument at $(v, y_T) \in C(K) \times C(K)$ if 
the limit
\[
\partial_1 \LL(v, y_T; h) := \lim_{s \to 0^+} \frac{\LL(v + s h, y_T) - \LL(v, y_T)}{s} \in \R
\]
exists for all $h \in C(K)$ and 
if the map  $\partial_1 \LL(v, y_T; \cdot)\colon C(K) \to \R$, $h \mapsto \partial_1 \LL(v, y_T;h)$,
is linear and continuous, i.e., an element of the topological dual space of $C(K)$. In this case, 
$\partial_1 \LL(v, y_T):= \partial_1 \LL(v, y_T; \cdot) \in \MM(K)$ is called the 
partial Gâteaux derivative of $\LL$ at $(v, y_T)$ w.r.t.\ the first argument, cf.\ 
\cite[section 2.2.1]{BonnansShapiro2000}.
As usual,
a local minimum of \eqref{eq:P} is a point $\bar \alpha \in D$ that satisfies 
\[
 \LL(\Psi(\alpha), y_T) \geq  \LL(\Psi(\bar \alpha), y_T),\qquad \forall \alpha \in B_r^D(\bar \alpha),
\]
for some $r > 0$. If $r$ can be chosen as $+\infty$, then we call 
$\bar \alpha$ a global minimum of \eqref{eq:P}. For a local minimum that is not 
a global minimum, we use the term \emph{spurious local minimum}. We would like to point 
out that we will not discuss the existence of global minima of \eqref{eq:P} in this paper.
In fact, it is easy to construct examples in which \eqref{eq:P} does not admit any global solutions, cf.\ \cite{Petersen2021}.
We will  focus entirely on the existence of spurious local minima that may prevent optimization algorithms
from producing a minimizing sequence for \eqref{eq:P}, i.e., 
a sequence $\{\alpha_k\}_{k=1}^\infty \subset D$ satisfying 
\[
\lim_{k \to \infty}  \LL(\Psi(\alpha_k), y_T) = \inf_{\alpha \in D}  \LL(\Psi(\alpha), y_T).
\]

For later use, we recall that the Hausdorff dimension $\dim_\HH(E)$ of a set $E \subset \R^m$ is defined by 
\[
\dim_\HH(E) := \inf \left \{ s \in [0, \infty) \;\big |\; \HH_s(E) = 0 \right \},
\]
where $\HH_s(E)$ denotes the $s$-dimensional Hausdorff outer measure 
\begin{equation}
\label{eq:Hausdorffoutermeasure}
\HH_s(E) := \lim_{\epsilon \to 0^+} \left ( \inf \left \{ \sum_{l=1}^\infty \diam(E_l)^s \; \Bigg | \; 
E \subset \bigcup_{l =1}^\infty E_l,\,\diam(E_l) < \epsilon \right \} \right ).
\end{equation}
Here, $\diam(\cdot)$ denotes the diameter of a set and 
the infimum on the right-hand side of \eqref{eq:Hausdorffoutermeasure} is taken over the set of covers $\{E_l\}_{l=1}^\infty$;
see \cite[Sections 3.5, 3.5.1c]{DiBenedetto2016}. Note that the Hausdorff dimension 
of a subspace is identical to the ``usual'' dimension of the subspace in the sense of 
linear algebra. In particular, we have $\dim(D) = \dim_\HH(D) = m$.

\section{Main results on the existence of spurious local minima}%
\label{sec:3}%
With the notation in place, 
we are in the position to formulate our main results on the existence of spurious 
local minima in training problems for neural networks whose activation functions possess
an affine segment. To be  precise, we state
our main observation in the form of two theorems -- one for activation functions
with a nonconstant affine segment and one for activation functions
with a constant segment. 

\begin{theorem}[case I: activation functions with a nonconstant affine segment]%
\label{th:spuriousminima-nonconstant}%
Let $K \subset \R^d$, $d \in \mathbb{N}$,
be a nonempty compact set and let 
$\psi \colon D \times \R^d \to \R$ be a neural network
with depth $L \in \mathbb{N}$, widths 
$w_i \in \mathbb{N}$, $i=0,...,L+1$, and nonpolynomial continuous activation functions
$\sigma_i\colon \R \to \R$, $i=1,...,L$, as in \eqref{eq:network}. Assume that:
\begin{enumerate}[label=\roman*)]
\item\label{th:spur1:item:i} 
$w_i \geq 2$ holds for all $i=1,...,L$.
\item\label{th:spur1:item:ii}   
$\sigma_i$ is affine and nonconstant on an open interval $I_i \neq \emptyset$ for all $i=1,...,L$.
\item\label{th:spur1:item:iii}  
$y_T \in C(K)$ is nonaffine, i.e., $\nexists (a,c) \in \R^d \times \R \colon y_T(x) = a^\top x + c,\;\forall x \in K$.\pagebreak
\item\label{th:spur1:item:iv} 
 $\LL\colon C(K) \times C(K)\to \R$ 
 is Gâteaux differentiable in its first argument with a nonzero partial derivative 
 at all points $(v, y_T) \in C(K) \times C(K)$ with $v \neq y_T$.
\item\label{th:spur1:item:v} 
 $\LL$ and $y_T$ are such that there exists a global solution $(\bar a, \bar c)$ of the problem 
\[
		\text{Minimize}\quad\LL( z_{a,c}, y_T)
		\quad~~~ \text{w.r.t.}\quad (a, c) \in \R^d \times \R
		\quad~~~ \text{s.t.}\quad z_{a,c}(x) = a^\top x + c.
\]
\end{enumerate}
Then there exists a set $E \subset D$ of Hausdorff dimension $\dim_\HH(E) \geq m - d - 1$
such that all elements of $E$ are spurious local minima of the training problem
\begin{equation*}
\tag{P}
		\text{Minimize} \quad  \LL(\Psi(\alpha), y_T)\qquad \text{w.r.t.}\quad \alpha \in D
\end{equation*}
and such that it holds
\[
\LL(\Psi(\alpha), y_T) =  \min_{(a,c) \in \R^d \times \R} \LL(z_{a,c}, y_T),\qquad \forall \alpha \in E.
\]
\end{theorem}

\begin{theorem}[case II: activation functions with a constant segment]%
\label{th:spuriousminima-constant}%
Suppose that $K \subset \R^d$, $d \in \mathbb{N}$,
is a nonempty compact set and let 
$\psi\colon D \times \R^d \to \R$ be a neural network
with depth $L \in \mathbb{N}$, widths 
$w_i \in \mathbb{N}$, $i=0,...,L+1$, and nonpolynomial continuous activation functions
$\sigma_i\colon \R \to \R$, $i=1,...,L$, as in \eqref{eq:network}. Assume that:
\begin{enumerate}[label=\roman*)]
\item\label{th:spur2:item:i}    
$\sigma_j$ is constant on an open interval $I_j \neq \emptyset$ for some  $j \in \{1,...,L\}$.
\item\label{th:spur2:item:ii}    
$y_T \in C(K)$ is nonconstant, i.e., $\nexists c \in \R \colon y_T(x) = c,\;\forall x \in K$.
\item\label{th:spur2:item:iii}    
$\LL\colon C(K) \times C(K)\to \R$ 
 is Gâteaux differentiable in its first argument with a nonzero partial derivative 
 at all points $(v, y_T) \in C(K) \times C(K)$ with $v \neq y_T$.
\item\label{th:spur2:item:iv}    
$\LL$ and $y_T$ are such that there exists a global solution $\bar c$ of the problem 
\[
		\text{Minimize}\quad\LL( z_{c}, y_T)
		\quad~~~ \text{w.r.t.}\quad c \in \R
		\quad~~~ \text{s.t.}\quad z_{c}(x) = c. 
\]
\end{enumerate}
Then there exists a set $E \subset D$ of Hausdorff dimension $\dim_\HH(E) \geq m - 1$
such that all elements of $E$ are spurious local minima of the training problem
\begin{equation*}
\tag{P}
		\text{Minimize} \quad  \LL(\Psi(\alpha), y_T)\qquad \text{w.r.t.}\quad \alpha \in D
\end{equation*}
and such that it holds
\[
\LL(\Psi(\alpha), y_T) =  \min_{c \in  \R} \LL(z_{c}, y_T),\qquad \forall \alpha \in E.
\]
\end{theorem}

The proofs of \cref{th:spuriousminima-nonconstant,th:spuriousminima-constant}
rely on geometric properties of the image $\Psi(D)$ of the function 
$\Psi\colon D \to C(K)$ in \eqref{eq:Psidef} and are carried out in \cref{sec:4}, 
see \cref{th:nohalfspaces,lem:aff1,lem:aff2,lem:aff3,lem:aff4,}.
Before we discuss them in detail, we give some remarks on the applicability and scope 
of \cref{th:spuriousminima-nonconstant,th:spuriousminima-constant}. 

First of all, we would like to point out that -- as far as continuous activation functions 
with an affine segment are concerned -- the assumptions on the maps $\sigma_i$ in
\cref{th:spuriousminima-nonconstant,th:spuriousminima-constant} are optimal.
The only continuous $\sigma_i$ that are locally affine and not covered by 
\cref{th:spuriousminima-nonconstant,th:spuriousminima-constant}  are globally affine functions
and for those it has been proven in \cite{Kawaguchi2016} that spurious local minima do not exist 
so that relaxing the assumptions on $\sigma_i$ in 
\cref{th:spuriousminima-nonconstant,th:spuriousminima-constant} 
in this direction is provably impossible. 
Compare also with \cite{Eftekhari2020,Laurent2018,Saxe2013,Yun2018,Zou2020} in this context.
Note that 
\cref{th:spuriousminima-nonconstant} covers in particular 
neural networks that involve an arbitrary mixture of 
PLU-,
ISRLU-, ELU-, ReLU-, and leaky/parametric ReLU-activations and that \cref{th:spuriousminima-constant}
applies, for instance, to neural networks with a ReLU- or an SQNL-layer; see \cite{Nicolae2018} and 
\cite[Corollary 40]{Christof2021} for the definitions of these functions. 
Because of this, the assertions of \cref{th:spuriousminima-nonconstant,th:spuriousminima-constant} 
hold in many situations arising in practice. 

Second, we remark that \cref{th:spuriousminima-nonconstant,th:spuriousminima-constant}  
can be rather easily extended
to neural networks with a vectorial output. For such networks,
the assumptions on the widths $w_i$ in point \ref{th:spur1:item:i} of \cref{th:spuriousminima-nonconstant}
have to be adapted depending on the in- and output dimension, but the basic 
ideas of the proofs remain the same, cf.\ the analysis 
of \cite{He2020} and the proof of \cite[Corollary 47]{Christof2021}. 
In particular, the arguments that we use in \cref{sec:4} to establish 
that the local minima in $E$ are indeed spurious carry over immediately.  
Similarly, it is also possible to extend the ideas presented in this paper to residual neural networks.
To do so, one can exploit that networks with skip connections can emulate 
classical multilayer perceptron architectures of the type \eqref{eq:network} on the training set $K$ 
by saturation, cf.\ \cite[proof of Corollary 52]{Christof2021}, and that skip connections do not impair the ability of a 
network with locally affine activation functions to emulate an affine linear mapping,
cf.\ the proofs of \cref{lem:aff1,lem:aff2}.
We omit discussing these generalizations in detail here to simplify the presentation.

Regarding the assumptions on $\LL$, it should be noted that 
the conditions in points \ref{th:spur1:item:iv} and \ref{th:spur1:item:v} of \cref{th:spuriousminima-nonconstant}
and points \ref{th:spur2:item:iii} and \ref{th:spur2:item:iv} of \cref{th:spuriousminima-constant}
are not very restrictive. The assumption that the partial Gâteaux derivative $\partial_1 \LL(v, y_T)$
is nonzero for $v \neq y_T$ simply expresses that the map $\LL(\cdot, y_T)\colon C(K) \to \R$
should not have any stationary points away from $y_T$.
This is a reasonable thing to assume since the purpose of the loss function is to 
measure the deviation from $y_T$ so that stationary points away from $y_T$ are not sensible. 
In particular, this assumption is automatically satisfied if $\LL$ has the form 
$\LL(v, y_T) = \FF(v - y_T)$ with a convex function $\FF\colon C(K) \to [0, \infty)$
that is Gâteaux differentiable in $C(K) \setminus \{0\}$ and satisfies $\FF(v) = 0$ iff $v= 0$. 
Similarly, the assumptions on the existence of 
the minimizers $(\bar a, \bar c)$ and $\bar c$ 
in  \cref{th:spuriousminima-nonconstant,th:spuriousminima-constant}
simply express that 
there should exist an affine linear/constant best approximation for $y_T$ 
w.r.t.\ the notion of approximation quality encoded in $\LL$. This condition is, for instance, 
satisfied when restrictions of the map $\LL( \cdot, y_T)\colon C(K) \to \R$ 
to finite-dimensional subspaces of $C(K)$
are radially unbounded and lower semicontinuous. A prototypical class 
of functions $\LL$ that satisfy all of the above conditions 
are tracking-type functionals in reflexive Lebesgue spaces
as the following lemma shows.

\begin{lemma}
\label{lem:LptrackingFunctional}
Let $K \subset \R^d$ be nonempty and compact, let $\mu \in \MM_+(K)$
be a measure whose support is equal to $K$, and let $1 < p < \infty$ be given. 
Define 
\begin{equation}
\label{eq:muloss}
\LL\colon C(K) \times C(K) \to [0, \infty),\qquad \LL(v, y_T) := \int_K |v  - y_T|^p \dd\mu.
\end{equation}
Then the function 
$\LL$ satisfies the assumptions \ref{th:spur1:item:iv} and \ref{th:spur1:item:v} of \cref{th:spuriousminima-nonconstant}
and the assumptions \ref{th:spur2:item:iii} and \ref{th:spur2:item:iv} of \cref{th:spuriousminima-constant}
for all  $y_T \in C(K)$.
\end{lemma}
\begin{proof}
From the dominated convergence theorem \cite[Theorem 3.3.2]{Benedetto2010},
it follows that $\LL\colon C(K) \times C(K) \to [0, \infty)$ is Gâteaux differentiable everywhere with 
\begin{equation}
\label{eq:randomeq364545}
\left \langle \partial_1 \LL(v, y_T), z\right \rangle_{C(K)} = \int_K p \sgn(v - y_T) |v - y_T|^{p-1} z \dd \mu,\qquad \forall v, y_T, z \in C(K).
\end{equation}
Since $C(K)$ is dense in $L_{\mu}^q(K)$
for all $1 \leq q < \infty$ by \cite[Proposition 7.9]{Folland1999}, 
\eqref{eq:randomeq364545} yields 
\begin{equation}
\label{eqrandomeq3636}
\partial_1 \LL(v, y_T) = 0 \in \MM(K)
\qquad 
\iff
\qquad
\int_K p |v - y_T|^{p-1} \dd \mu = 0.
\end{equation}
Due to the continuity of the function $|v - y_T|^{p-1}$ and since the assumptions on $\mu$
imply that 
$\mu(O) > 0$ holds for all 
$O \in \tau_K \setminus \{\emptyset\}$, 
the right-hand side  of
\eqref{eqrandomeq3636} can only be true if $v - y_T$ is the zero function in $C(K)$, i.e., if $v = y_T$.
This shows that $\LL$ indeed satisfies 
condition \ref{th:spur1:item:iv} in \cref{th:spuriousminima-nonconstant}
and  condition \ref{th:spur2:item:iii}  in \cref{th:spuriousminima-constant}
for all $y_T \in C(K)$. To see that $\LL$ also satisfies assumption 
\ref{th:spur1:item:v} of \cref{th:spuriousminima-nonconstant}
and assumption \ref{th:spur2:item:iv} of \cref{th:spuriousminima-constant}, it suffices to note that 
$\|\cdot\|_{L^p_\mu(K)}$ defines a norm on $C(K)$ due to the assumptions on $\mu$.
This implies that restrictions of the map $\LL( \cdot, y_T)\colon C(K) \to \R$ to 
finite-dimensional subspaces of $C(K)$ are continuous and radially unbounded for all arbitrary but fixed $y_T \in C(K)$
and that the theorem of Weierstrass can be used to establish the
existence of the minimizers $(\bar a, \bar c)$ and $\bar c$ in 
points \ref{th:spur1:item:v} and \ref{th:spur2:item:iv} of \cref{th:spuriousminima-nonconstant,th:spuriousminima-constant},
respectively.
\end{proof}

Note that,
in the case $\mu =\frac{1}{n} \sum_{k=1}^n  \delta_{x_k}$,
$K = \{x_1,...,x_n\} \subset \R^d$, $d \in \mathbb{N}$, $n \in \mathbb{N}$,
i.e., in the situation where $\mu$ is the normalized sum of $n$ Dirac measures 
supported at points $x_k \in \R^d$, $k=1,...,n$,
a problem \eqref{eq:P} with a loss function of the form \eqref{eq:muloss} 
can be recast as
\begin{equation}
\label{eq:plossfinite}
		\text{Minimize} \quad \frac{1}{n} \sum_{k=1}^n |\psi(\alpha,x_k) - y_T(x_k) |^p  \qquad \text{w.r.t.}\quad \alpha \in D.
\end{equation}
In particular, for $p=2$, one recovers a classical squared loss problem 
with a finite number of data samples.
This shows that our results
indeed extend \cite[Theorem~1]{Yun2018}, where the assertion of \cref{th:spuriousminima-nonconstant}
was proven for 
finite-dimensional squared loss training problems for one-hidden-layer neural networks 
with activation functions of parameterized ReLU-type. 
Compare also with \cite{Christof2021,Ding2020,He2020,Liu2021}  in this context.
Another natural choice for $\mu$ in \eqref{eq:muloss} is the restriction of 
the Lebesgue measure to the Borel sigma-algebra of the closure $K$
of a nonempty bounded open set $\Omega \subset \R^d$. For this choice, \eqref{eq:P}
becomes a standard $L^p$-tracking-type problem as often considered in the field 
of optimal control, cf.\ \cite{Christof2021-2} and the references therein.
A further interesting example is the case $K = \closure_{\R^d}(\{x_k\}_{k=1}^\infty)$ and 
$\mu = \sum_{k=1}^\infty c_k \delta_{x_k}$ 
involving a bounded sequence of points
$\{x_k\}_{k=1}^\infty \subset \R^d$ and weights $\{c_k\}_{k=1}^\infty \subset (0, \infty)$ with 
\mbox{$\sum_{k=1}^\infty c_k < \infty$}.
Such a measure $\mu$ gives rise to a training problem in an intermediate regime 
between the finite and continuous sampling case. 

We remark that, for problems of the type \eqref{eq:plossfinite} with $p=2$, it can be shown that 
the spurious local minima in \cref{th:spuriousminima-nonconstant,th:spuriousminima-constant} can be arbitrarily bad 
in the sense that they may yield loss values that are arbitrarily far away from the optimal one and 
may give rise to realization vectors $\{\psi(\alpha,x_k)\}_{k=1}^n$ that are arbitrarily far away
in relative and absolute terms
from every optimal realization vector of the network. 
For a precise statement of these results for finite-dimensional squared loss problems and the definitions of the related concepts,
we refer the reader to \cite[Corollary 47, Definition 3, and Estimates (39), (40)]{Christof2021}.
Similarly, it can also be proven that the appearance of spurious local minima in problems of the type 
\eqref{eq:plossfinite} can, in general, not
be avoided 
by adding a regularization term to the loss function that penalizes the size of the 
parameters in $\alpha$, see \cite[Corollary 51]{Christof2021}. We remark that 
the proofs used to establish these results in \cite{Christof2021} 
all make use of compactness arguments and homogeneity properties of $\LL$
and thus do not carry over to the general infinite-dimensional setting 
considered in \cref{th:spuriousminima-nonconstant,th:spuriousminima-constant}, cf.\ 
the derivation of \cite[Lemma 10]{Christof2021}.

As a final remark, we would like to point out that, in the degenerate case $n=1$, 
the training problem \eqref{eq:plossfinite} \emph{does not} possess any spurious local minima 
(as one may easily check by varying the bias $b_{L+1}$ on the last layer of $\psi$). This effect does not contradict 
our results since, for $n=1$, 
the set $K$ is a singleton, every $y_T \in C(K) \cong \R$ can be precisely fit with a constant function,
and  condition \ref{th:spur1:item:iii} in \cref{th:spuriousminima-nonconstant} and condition \ref{th:spur2:item:ii}
in \cref{th:spuriousminima-constant} are always violated. 
Note that this highlights that the assumptions of \cref{th:spuriousminima-nonconstant,th:spuriousminima-constant} are sharp.

\section{Nonexistence of supporting half-spaces and proofs of main results}
\label{sec:4}
In this section, we prove \cref{th:spuriousminima-nonconstant,th:spuriousminima-constant}.
The point of departure for our analysis is the following theorem of Pinkus.

\begin{theorem}[{\cite[Theorem 3.1]{Pinkus1999}}]%
\label{th:pinkus}%
Let $d \in \mathbb{N}$.
Let  $\sigma\colon \R \to \R$ be a nonpolynomial continuous function.
Consider the linear hull
\begin{equation}
\label{eq:Vdef}
V := 
\span 
\left \{
x \mapsto \sigma(a^\top x + c)
\; \big |\;
a \in \R^d, c \in \R
\right \} \subset C(\R^d).
\end{equation}
Then the set $V$ is dense in $C(\R^d)$ in the topology of uniform convergence on compacta. 
\end{theorem}

Note that, as $V$ contains precisely those functions that can be represented by 
one-hidden-layer neural networks of the type \eqref{eq:network} with $\sigma_1 = \sigma$,
the last theorem is nothing else than the universal approximation theorem in the arbitrary width case,
cf.\ \cite{Cybenko1989,Hornik1991}.
In other words, \cref{th:pinkus} simply expresses that, for every 
nonpolynomial $\sigma \in C(\R)$, every nonempty compact set $K \subset \R^d$,
every $y_T \in C(K)$, 
and every $\varepsilon > 0$, there exists a width 
$\tilde w_1 \in \mathbb{N}$ such that a neural network $\psi$  with the architecture in \eqref{eq:network},
depth \mbox{$L=1$}, width $w_1 \geq \tilde w_1$, and activation function $\sigma$ is able to approximate $y_T$
in \mbox{$(C(K), \|\cdot\|_{C(K)})$} up to the error $\varepsilon$. 
In what follows, we will not explore what \cref{th:pinkus} implies 
for the approximation capabilities of neural networks when the widths go to infinity
but rather which consequences the density of the space $V$ in \eqref{eq:Vdef} has 
for a given neural network with a fixed architecture. More precisely, we will 
use \cref{th:pinkus} to prove that the image $\Psi(D) \subset C(K)$ of the function
\mbox{$\Psi\colon D \to C(K)$} in \eqref{eq:Psidef}
does not admit any supporting half-spaces
when a neural network $\psi$ 
with non\-polynomial continuous activations $\sigma_i$
and arbitrary fixed dimensions $L, w_i \in \mathbb{N}$ is considered.

\begin{theorem}[nonexistence of supporting half-spaces]%
\label{th:nohalfspaces}%
Let $K \subset \R^d$, $d \in \mathbb{N}$, be a nonempty compact set
and let $\psi\colon D \times \R^d \to \R$ be a neural network 
with depth $L \in \mathbb{N}$, widths 
$w_i \in \mathbb{N}$, $i=0,...,L+1$, and continuous nonpolynomial activation functions
$\sigma_i\colon \R \to \R$, $i=1,...,L$, as in \eqref{eq:network}.
Denote with $\Psi\colon D \to C(K)$ the function in \eqref{eq:Psidef}.
Then a measure $\mu \in \MM(K)$ and a constant $c \in \R$ satisfy 
\begin{equation}
\label{eq:randomeq2535}
\left \langle \mu, z \right \rangle_{C(K)} \leq c,\qquad \forall z \in \Psi(D),
\end{equation}
if and only if $\mu = 0$ and $c \geq 0$. 
\end{theorem}

\begin{proof}
The implication ``$\Leftarrow$'' is trivial. 
To prove ``$\Rightarrow$'', we assume that $c \in \R$ and $\mu \in \MM(K)$
satisfying \eqref{eq:randomeq2535} are given. 
From the definition of $\Psi$, we obtain that 
$\beta \Psi(\alpha) \in \Psi(D)$ holds for all $\beta \in \R$ and all $\alpha \in D$.
If we exploit this property in \eqref{eq:randomeq2535}, then we obtain
that $c$ and $\mu$ have to satisfy $c \geq 0$ and 
\begin{equation}
\label{eq:muidentity}
\left \langle \mu, z \right \rangle_{C(K)} = 0,\qquad \forall z \in \Psi(D).
\end{equation}
It remains to prove that $\mu$ vanishes. To this end, 
we first reduce the situation to the case $w_1 = ... = w_L = 1$.
Consider  a parameter $\tilde \alpha \in D$
whose weights and biases have the form 
\begin{equation}
\label{eq:randomeq364564ge}
\begin{gathered}
\tilde A_1 := 
\begin{pmatrix}
a_1^\top
\\
0_{(w_1 - 1) \times d}
\end{pmatrix}
,
\qquad
\tilde A_i := 
\begin{pmatrix}
a_i\quad 0_{1 \times (w_{i-1} - 1)}
\\
0_{(w_i - 1)\times w_{i-1}}
\end{pmatrix}
,
~~i=2,...,L+1,
\\
\tilde b_i :=
\begin{pmatrix}
c_i
\\
0_{w_i-1}
\end{pmatrix}
,
~~
i=1,...,L+1,
\end{gathered}
\end{equation}
for some arbitrary but fixed $a_1 \in \R^{d}$, $a_i \in \R$, $i=2,...,L+1$, and $c_i \in \R$, $i=1,..., L+1$,
where $0_{p\times q} \in \R^{p \times q}$ and $0_p \in \R^p$ denote the zero matrix and zero vector 
in $\R^{p \times q}$ and $\R^p$, $p,q \in \mathbb{N}$, respectively, with the convention that these 
zero entries are ignored in the case $p=0$ or $q = 0$.
For such a parameter $\tilde \alpha$, we obtain from \eqref{eq:network} that 
\[
\psi(\tilde \alpha, x ) = \left ( \theta_{L+1}^{a_{L+1}, c_{L+1}} \circ ... 
\circ  \theta_{1}^{a_1, c_1}  \right )
\left (  x  \right ),
\qquad \forall x \in \R^d,
\]
holds with the functions $\theta_1^{a_1,c_1}\colon \R^d \to \R$, $\theta_i^{a_i, c_i}\colon \R \to \R$, $i=2,...,L+1$, given by 
\begin{equation*}
\begin{gathered}
\theta_1^{a_1, c_1}(z) :=  \sigma_1\left (a_1^\top z + c_1 \right ),
\qquad 
\theta_i^{a_i, c_i}(z) := \sigma_i\left (a_i z + c_i \right ),~\forall i=2,...,L,
\\
\theta_{L+1}^{a_{L+1}, c_{L+1}}(z) := a_{L+1} z + c_{L+1}.
\end{gathered}
\end{equation*}
In combination with \eqref{eq:muidentity} and the definition of $\Psi$, this yields that
\begin{equation}
\label{eq:randomeq63636}
\left \langle
\mu,  \theta_{L+1}^{a_{L+1}, c_{L+1}} \circ ... 
\circ  \theta_{1}^{a_1, c_1}
\right \rangle_{C(K)}
=
\int_K 
\left ( \theta_{L+1}^{a_{L+1}, c_{L+1}} \circ ... 
\circ  \theta_{1}^{a_1, c_1}  \right )
\left (  x  \right )
\dd \mu(x) = 0 
\end{equation}
holds for all $a_i, c_i$, $i=1,...,L+1$. 
Next, we use \cref{th:pinkus} to reduce the number of layers in \eqref{eq:randomeq63636}.
Suppose that $L > 1$ holds
and let $a_i,c_i$, $i \in \{1,...,L+1\} \setminus \{L\}$, be arbitrary but fixed parameters. 
From the compactness of $K$
and the continuity of the function
$K \ni x \mapsto \left ( \theta_{L-1}^{a_{L -1}, c_{L-1}} \circ ... \circ  \theta_{1}^{a_1, c_1} \right ) (x) \in \R$,
we obtain that the image $F := \left ( \theta_{L-1}^{a_{L -1}, c_{L-1}} \circ ... \circ  \theta_{1}^{a_1, c_1} \right )(K) \subset \R$
is compact, and from \cref{th:pinkus}, it follows 
that there exist numbers $n_l \in \mathbb{N}$ and $\beta_{k,l}, \gamma_{k,l},  \lambda_{k,l}\in \R$, $k=1,...,n_l$,
$l \in \mathbb{N}$, such that the sequence of continuous functions 
\[
\zeta_l\colon F \to \R,\qquad z \mapsto \sum_{k=1}^{n_l} \lambda_{k,l} \sigma_L( \beta_{k,l} z + \gamma_{k,l}), 
\]
converges uniformly on $F$ to the identity map for $l \to \infty$.
Since \eqref{eq:randomeq63636} holds for all choices of parameters, 
we further know that 
\begin{equation*}
\int_K 
a_{L+1} \lambda_{k,l}
\sigma_L
\left (
\beta_{k,l}
\left ( \theta_{L-1}^{a_{L -1}, c_{L-1}} \circ ... \circ  \theta_{1}^{a_1, c_1} \right ) (x)
+
\gamma_{k,l}
\right )
+ \frac{1}{n_l} c_{L+1}
\dd \mu(x) = 0
\end{equation*}
holds for all $k = 1,...,n_l$ and all $l \in \mathbb{N}$. Due to the linearity of the integral, 
we can add all of the above equations to obtain that 
\begin{equation*}
\int_K 
a_{L+1} \zeta_l
\left [
\left ( \theta_{L-1}^{a_{L -1}, c_{L-1}} \circ ... \circ  \theta_{1}^{a_1, c_1} \right ) (x) 
\right ]
+ c_{L+1}
\dd \mu(x) = 0, \qquad \forall l \in \mathbb{N},
\end{equation*}
holds and, after passing to the limit $l \to \infty$ by means of the dominated convergence theorem, that
\begin{equation*}
\int_K 
a_{L+1} 
\left ( \theta_{L-1}^{a_{L -1}, c_{L-1}} \circ ... \circ  \theta_{1}^{a_1, c_1} \right ) (x) 
+ c_{L+1}
\dd \mu(x) = 0. 
\end{equation*}
Since $a_i,c_i$, $i \in \{1,...,L+1\} \setminus \{L\}$, were arbitrary, this is precisely 
\eqref{eq:randomeq63636} with the $L$-th layer removed. By proceeding 
iteratively along the above lines, it follows that $\mu$ satisfies
\begin{equation*}
\int_K 
a_{L+1} 
\sigma_1(a_1^\top x + c_1)
+ c_{L+1}
\dd \mu(x) = 0
\end{equation*}
for all $a_{L+1}, c_{L+1}, c_1 \in \R$ and all $a_1 \in \R^d$. 
Again by the density in \eqref{th:pinkus} and the linearity of the integral, 
this identity can only be true if $\left \langle \mu, z \right \rangle_{C(K)} = 0$ holds 
for all $z \in C(K)$. Thus, $\mu = 0$ and the proof is complete. 
\end{proof}\pagebreak

\begin{remark}~
\begin{itemize}
\item
\Cref{th:nohalfspaces} is, in fact,
equivalent to \cref{th:pinkus}. Indeed,  the implication ``\cref{th:pinkus} $\Rightarrow$ \cref{th:nohalfspaces}''
has been proven above. To see that \cref{th:nohalfspaces} implies \cref{th:pinkus},
one can argue by contradiction. If the space $V$ in \eqref{eq:Vdef} is not dense 
in $C(\R^d)$ in the topology of uniform convergence on compacta, then there exist
a nonempty compact set $K \subset \R^d$ and a nonzero $\mu \in \MM(K)$
such that $\left \langle \mu, v\right \rangle_{C(K)} = 0$ holds for all $v \in V$,
cf.\ the proof of \cite[Proposition 3.10]{Pinkus1999}. Since \cref{th:nohalfspaces}
applies to networks with $L=1$ and $w_1 = 1$,  the variational identity 
 $\left \langle \mu, v\right \rangle_{C(K)} = 0$  for all $v \in V$
can only be true if $\mu = 0$. Hence, one arrives at a contradiction and the density 
in \cref{th:pinkus} follows. Compare also with the classical proofs of the universal approximation theorem
in \cite{Cybenko1989} and \cite{Hornik1991} in this context which prove results 
similar to \cref{th:nohalfspaces} as an intermediate step. 
In combination with the comments after \cref{th:pinkus},
this shows that the arguments that we use in the following to establish the 
existence of spurious local minima in training problems of the form \eqref{eq:P} are 
indeed closely related to the universal approximation property.

\item 
It is easy to check that the nonexistence of supporting half-spaces in \cref{th:nohalfspaces}
 implies that, for every finite training set $K = \{x_1,...,x_n\}$
 and every network $\psi$ with associated
 function $\Psi\colon D \to C(K)\cong \R^n$ satisfying the assumptions 
of \cref{th:nohalfspaces}, we have
\begin{equation}
\label{eq:improvedexpressiveness}
\sup_{y_T \in \R^n \colon |y_T| = 1} \inf_{y \in \Psi(D)} |y - y_T|^2 < 1.
\end{equation}
This shows that \cref{th:nohalfspaces} implies the ``improved expressiveness''-condition in 
\cite[Assumption 6-II)]{Christof2021} and may be used to establish an alternative proof of 
\cite[Theorem 39, Corollary 40]{Christof2021}. We remark that, for infinite $K$, a condition analogous 
to \eqref{eq:improvedexpressiveness} cannot be expected to hold for a neural network. 
In our analysis, \cref{th:nohalfspaces} serves as a substitute for 
 \eqref{eq:improvedexpressiveness} that remains 
true in the infinite-dimensional setting and for arbitrary loss functions.
\end{itemize}
\end{remark}

We are now in the position to prove \cref{th:spuriousminima-nonconstant,th:spuriousminima-constant}.
We begin by constructing the sets of local minima $E \subset D$ that appear in these theorems.
As before, we distinguish between activation functions 
with a nonconstant affine segment and activation functions with a constant segment.

\begin{lemma}
\label{lem:aff1}
Consider a nonempty compact set $K \subset \R^d$ and a neural network $\psi\colon D \times \R^d \to \R$
with depth $L \in \mathbb{N}$, widths $w_i \in \mathbb{N}$, and continuous activation functions $\sigma_i$
as in \eqref{eq:network}. Suppose that $\LL\colon C(K) \times C(K) \to \R$ and $y_T \in C(K)$
are given such that $\LL$, $y_T$, and the functions $\sigma_i$ satisfy 
the conditions \ref{th:spur1:item:ii} and \ref{th:spur1:item:v} 
in \cref{th:spuriousminima-nonconstant}. 
Then there exists a set 
$E \subset D$ of Hausdorff dimension $\dim_\HH(E) \geq m - d - 1$
such that 
all elements of $E$ are local minima of \eqref{eq:P} and such that 
\begin{equation}
\label{eq:randomeq2672ged63}
\LL(\Psi(\alpha), y_T) = \min_{(a,c) \in \R^d \times \R} \LL(z_{a,c}, y_T)
\end{equation}
holds for all $\alpha \in E$, where $z_{a,c}$ is defined by $z_{a,c}(x) := a^\top x + c$ for all $x \in \R^d$.
\end{lemma}

\begin{proof}
Due to \ref{th:spur1:item:ii}, we can find numbers 
$c_i \in \R$, $\varepsilon_i > 0$, $\beta_i \in \R \setminus \{0\}$, and $\gamma_i \in \R$
such that $\sigma_i(s) = \beta_i s + \gamma_i$ holds for all $s \in I_i = (c_i - \varepsilon_i, c_i + \varepsilon_i)$
and all $i=1,...,L$,
and from  \ref{th:spur1:item:v}, we obtain that there exist $\bar a \in \R^d$ and $\bar c \in \R$ satisfying 
\[
\LL(z_{a,c}, y_T) \geq \LL(z_{\bar a, \bar c}, y_T),\qquad \forall (a,c) \in \R^d \times \R. 
\]
Consider now the parameter $\bar \alpha = 
\{ (\bar A_{i}, \bar b_{i})\}_{i=1}^{L+1} \in D$ 
whose weights and biases are given by 
\begin{equation}
\label{eq:baralpha}
\begin{gathered}
\bar A_1 := 
\frac{\varepsilon_1}{2\max_{u \in K} |\bar a^\top u| + 1}
\begin{pmatrix}
\bar a^\top
\\
0_{(w_1 - 1)\times w_0}
\end{pmatrix}
\in \R^{w_1 \times w_{0}},
\\
\bar b_1 := 
c_1 1_{w_1}
\in \R^{w_1},
\\
\bar A_i := 
\frac{\varepsilon_i}{\beta_{i-1} \varepsilon_{i-1}}
\begin{pmatrix}
1\quad 0_{1 \times (w_{i - 1} - 1)}
\\
0_{(w_i - 1)\times  w_{i - 1}}
\end{pmatrix}
\in \R^{w_i \times w_{i - 1}},\qquad i=2,...,L,
\\
\bar b_i := 
c_i 1_{w_i}
 - (c_{i - 1}\beta_{i - 1} + \gamma_{i - 1}) \bar A_i 1_{w_{i - 1}}
\in \R^{w_i},\qquad i=2,...,L,
\\
\bar A_{L+1}:= 
\frac{2\max_{u \in K} |\bar a^\top u| + 1}{\beta_{L} \varepsilon_{L}}
\begin{pmatrix}
1~~0_{1 \times (w_{L} - 1)}
\end{pmatrix}
\in \R^{w_{L+1} \times w_{L}},
\\
\bar b_{L+1} := 
\bar c - (c_{L}\beta_L + \gamma_{L}) \bar A_{L+1} 1_{w_{L}}
\in \R^{w_{L+1}},
\end{gathered}
\end{equation}
where the symbols $0_{p\times q} \in \R^{p \times q}$ and $0_p \in \R^p$ again denote zero matrices and zero vectors,
respectively, with the same conventions as before and where $1_p \in \R^p$ denotes 
a vector whose entries are all one. 
Then it is easy to check by induction that, for all $x \in K$, we have
\begin{equation}
\label{eq:randomeq273535}
\begin{gathered}
\bar A_1 x + \bar b_1 
=
\frac{\varepsilon_1}{2\max_{u \in K} |\bar a^\top u| + 1}
\begin{pmatrix}
\bar a^\top x
\\
0_{w_{1} - 1}
\end{pmatrix}
+
c_1 1_{w_1}
\in \left (c_1 -  \varepsilon_1,  c_1 +  \varepsilon_1  \right )^{w_1},
\\
\bar A_{i}\left (
\varphi_{i-1}^{\bar A_{i-1}, \bar b_{i-1}} \circ ... \circ \varphi_{1}^{\bar A_{1}, \bar b_{1}}(x)
\right )
+
\bar b_i
=
\frac{\varepsilon_i}{2\max_{u \in K} |\bar a^\top u| + 1}
\begin{pmatrix}
\bar a^\top x
\\
0_{w_{i} - 1}
\end{pmatrix}
+
c_i 1_{w_i}
\\
\qquad\qquad\qquad\qquad\qquad\qquad\qquad\qquad\qquad 
\in \left (c_i - \ \varepsilon_i,  c_i + \varepsilon_i  \right )^{w_i},
\quad \forall i=2,...,L,
\end{gathered}
\end{equation}
and
\begin{equation*}
\psi(\bar \alpha, x) =
\left (\varphi_{L+1}^{\bar A_{L+1}, \bar b_{L+1}} \circ ... \circ \varphi_{1}^{\bar A_{1}, \bar b_{1}}\right )(x)
=
\bar a^\top x  + \bar c. 
\end{equation*}
The parameter $\bar \alpha$
thus satisfies $\Psi(\bar \alpha) = z_{\bar a, \bar c} \in C(K)$.
Because of the compactness of $K$, the openness of the sets $ (c_i - \varepsilon_i,  c_i + \varepsilon_i) ^{w_i}$,
$i=1,...,L$,
 the affine-linearity of $\sigma_i$ on $(c_i - \varepsilon_i, c_i + \varepsilon_i)$,
 and
 the continuity of the functions $D \times \R^d  \ni (\alpha, x) \mapsto A_1 x + b_1 \in \R^{w_1}$
and
$D \times \R^d  \ni (\alpha, x) \mapsto A_i (\varphi_{i-1}^{A_{i-1}, b_{i-1}} \circ ... \circ \varphi_{1}^{A_{1},b_{1}}(x)) + b_i  \in \R^{w_i}$,
$i=2,...,L$,
it follows that there exists $r > 0$ such that 
all of the inclusions in \eqref{eq:randomeq273535} remain valid for 
$x \in K$ and $\alpha \in B_r^D(\bar \alpha)$ and such that 
$\Psi(\alpha) \in C(K)$ is affine (i.e., of the form $z_{a,c}$) 
for all $\alpha \in B_r^D(\bar \alpha)$.
As $z_{\bar a, \bar c}$ is the global solution of the best approximation problem in \ref{th:spur1:item:v},
this shows that $\bar \alpha$ is a local minimum of \eqref{eq:P} that satisfies \eqref{eq:randomeq2672ged63}.

To show that there are many such local minima,
we require some additional notation. 
Henceforth, with $a_1,...,a_{w_1} \in \R^d$ we denote the row vectors 
in the weight matrix $A_1$ and with $e_1,...,e_{w_1} \in \R^{w_1}$ the standard basis vectors of $\R^{w_1}$.
We further introduce the abbreviation $\alpha'$ 
for the collection of all parameters of $\psi$ that belong to the degrees of freedom 
$A_{L+1},...,A_2, b_L,...,b_1$, and  $a_2,...,a_{w_1}$. 
The space of all such $\alpha'$ is denoted by $D'$. Note that this space 
has dimension $\dim(D') = m - d - 1 > 0$. We again endow $D'$ with the
Euclidean norm of the space $\R^{m - d - 1}$ that $D'$ can be transformed into by reordering 
the entries in $\alpha'$. 
As before, in what follows, a bar indicates that 
we refer to the parameter $\bar \alpha \in D$ constructed in \eqref{eq:baralpha}, 
i.e., $\bar a_k$ refers to the $k$-th row of $\bar A_1$, $\bar \alpha' \in D'$ refers 
to $(\bar A_{L+1},...,\bar A_2, \bar b_L,...,\bar b_1, \bar a_2,..., \bar a_{w_1})$, etc. 

To construct a set $E \subset D$ as in the assertion of the lemma, we first note that 
the local affine linearity of $\sigma_i$, the definition of $\bar \alpha$, our choice of $r>0$, 
and the architecture of $\psi$ imply that 
there exists a 
continuous function $\Phi\colon D' \to \R$ which satisfies 
$\Phi(\bar \alpha') + \bar b_{L+1} = \bar c$ and 
\begin{equation}
\label{eq:randomeq2635}
\psi(\alpha, x)
= \left ( \prod_{i=1}^{L} \beta_i \right ) \left (  A_{L+1}A_L ... A_1\right )x
+
\Phi(\alpha') + b_{L+1}
\end{equation}
for all $x \in K$ and all $\alpha = \{ (A_{i}, b_{i})\}_{i=1}^{L+1} \in B_r^D(\bar \alpha)$, cf.\ \eqref{eq:randomeq273535}.
Define
\begin{equation*}
\Theta\colon D' \to \R^d,
\qquad
\Theta(\alpha')
:=
\left ( \prod_{i=1}^{L} \beta_i \right ) 
\left [\left (  A_{L+1}A_L ... A_2\right )
\begin{pmatrix}
0 \\
a_2^\top
\\
\vdots
\\
a_{w_1}^\top
\end{pmatrix}
\right ]^\top,
\end{equation*}
and
\begin{equation*}
\Lambda\colon D' \to \R
,\qquad
\Lambda(\alpha')
:=
\left ( \prod_{i=1}^{L} \beta_i \right ) \left (  A_{L+1}A_L ... A_2\right )e_1.
\end{equation*}
Then \eqref{eq:randomeq2635} can be recast as 
\begin{equation}
\label{eq:randomeq263535344}
\psi(\alpha, x)
= \Theta(\alpha')^\top x + \Lambda(\alpha') a_1^\top x + \Phi(\alpha') + 
b_{L+1},\quad \forall x \in K,\quad \forall \alpha  \in B_r^D(\bar \alpha).
\end{equation}
Note that, again by the construction of $\bar \alpha$ in \eqref{eq:baralpha}, we have $\Theta(\bar \alpha') = 0$,
$\Lambda(\bar \alpha')  \neq 0$, 
and $\Lambda(\bar \alpha')\bar a_1 = \bar a$. 
In particular, due to the continuity of 
$\Lambda\colon D' \to \R$, we can find $r' > 0$ such that $\Lambda(\alpha') \neq 0$ holds 
for all $\alpha' \in B_{r'}^{D'}(\bar \alpha')$. This allows us to define
\[
g_1\colon B_{r'}^{D'}(\bar \alpha') \to \R^d,
\qquad 
g_1(\alpha') := \frac{\Lambda(\bar \alpha')}{\Lambda(\alpha')}\bar a_1 - \frac{\Theta(\alpha')}{\Lambda(\alpha')}, 
\]
and 
\[
g_2\colon B_{r'}^{D'}(\bar \alpha') \to \R,
\qquad
g_2(\alpha') := \bar c  - \Phi(\alpha'). 
\]
By construction, these functions $g_1$ and $g_2$ are continuous and satisfy 
$g_1(\bar \alpha') = \bar a_1$, $g_2(\bar \alpha') = \bar b_{L+1}$, and 
\begin{equation}
\label{eq:randomeq263535344-2}
\Theta(\alpha')^\top x + \Lambda(\alpha') g_1(\alpha')^\top x + \Phi(\alpha') + 
g_2(\alpha') = \bar a^\top x + \bar c
\end{equation}
for all $\alpha' \in B_{r'}^{D'}(\bar \alpha')$ and all $x \in \R^d$. Again due to the continuity, 
this implies that, after possibly making $r'$ smaller, we have 
\[
E := 
\left \{
\alpha \in D
\; \Big |\;
\alpha' \in B_{r'}^{D'}(\bar \alpha'),
a_1 = g_1(\alpha'),
b_{L+1} = g_2(\alpha')
\right \} \subset B_r^D(\bar \alpha).
\]
For all elements $\tilde \alpha$ of the resulting set $E$,
it now follows from 
\eqref{eq:randomeq263535344} and \eqref{eq:randomeq263535344-2}
that 
\begin{equation*}
\begin{aligned}
\psi(\tilde \alpha, x) 
&= 
 \Theta(\tilde \alpha')^\top x + \Lambda(\tilde \alpha') \tilde a_1^\top x + \Phi(\tilde \alpha') + 
\tilde b_{L+1}
\\
&=
\Theta(\tilde \alpha')^\top x + \Lambda(\tilde \alpha') g_1(\tilde \alpha')^\top x + \Phi(\tilde \alpha') + 
g_2(\tilde \alpha') = \bar a^\top x + \bar c,\qquad \forall x \in K. 
\end{aligned}
\end{equation*}
Thus, $\Psi(\tilde \alpha) = z_{\bar a, \bar c}$ and, 
due to the definitions of $r$, $\bar a$, and $\bar c$,
\[
\LL(\Psi(\tilde \alpha), y_T) = 
\LL(z_{\bar a, \bar c}, y_T) =
\min_{(a,c) \in \R^d \times \R} \LL(z_{a,c}, y_T)
=
\min_{\alpha \in B_r^D(\bar \alpha)} \LL(\Psi(\alpha), y_T) 
\]
for all $\tilde \alpha \in E \subset  B_r^D(\bar \alpha)$. 
This shows that all elements of $E$ are local minima of \eqref{eq:P}
that satisfy \eqref{eq:randomeq2672ged63}.
Since $E$ is, modulo reordering of the entries in $\alpha$, 
nothing else than the graph of a function 
defined on an open subset of  $\R^{m - d - 1}$ with values in $\R^{d + 1}$,
the fact that the Hausdorff dimension of $E$ in $D$ is at least $m - d - 1$
immediately follows from the choice of the norm on $D$ and 
classical results, see \cite[Corollary 8.2c]{DiBenedetto2016}. 
\end{proof}

\begin{lemma}
\label{lem:aff2}
Consider a nonempty compact set $K \subset \R^d$ and a neural network $\psi\colon D \times \R^d \to \R$
with depth $L \in \mathbb{N}$, widths $w_i \in \mathbb{N}$, and continuous activation functions $\sigma_i$
as in \eqref{eq:network}. Suppose that  $\LL\colon C(K) \times C(K) \to \R$ and  $y_T \in C(K)$
are given such that $\LL$, $y_T$, and the functions $\sigma_i$ satisfy 
the conditions \ref{th:spur2:item:i} and \ref{th:spur2:item:iv}
in \cref{th:spuriousminima-constant}. 
Then there exists a set $E \subset D$ of Hausdorff dimension $\dim_\HH(E) \geq m - 1$
such that 
all elements of $E$ are local minima of \eqref{eq:P} and such that 
\begin{equation}
\label{eq:randomeq26353-3}
\LL(\Psi(\alpha), y_T) = \min_{c \in \R} \LL(z_{c}, y_T)
\end{equation}
holds for all $\alpha \in E$, where $z_{c}$ is defined by $z_{c}(x) := c$ for all $x \in \R^d$.
\end{lemma}

\begin{proof}
The proof of \cref{lem:aff2} is analogous to that of \cref{lem:aff1} but simpler:
From \ref{th:spur2:item:i}, we obtain that there exist an index
$j \in \{1,...,L\}$ and numbers $c_j \in \R$, $\varepsilon_j > 0$, and $\gamma_j \in \R$
such that $\sigma_j(s) = \gamma_j$ holds for all $s \in I_j = (c_j - \varepsilon_j, c_j + \varepsilon_j)$,
and from \ref{th:spur2:item:iv}, it follows that we can find a number $\bar c \in \R$
satisfying $\LL(z_c, y_T) \geq  \LL(z_{\bar c}, y_T)$ for all $c \in \R$. 
Define $\bar \alpha = \{ (\bar A_{i}, \bar b_{i})\}_{i=1}^{L+1}$ 
to be the element of $D$ whose weights and biases are given by 
\begin{equation*}
\begin{gathered}
\bar A_i := 0 \in \R^{w_i \times w_{i - 1}},~\forall i \in \{1,...,L+1\},
\qquad 
\bar b_i := 0 \in \R^{w_i},~\forall i \in \{1,...,L\} \setminus \{j\},
\\
\bar b_j := c_j 1_{w_j} \in \R^{w_j},\quad\text{and}\quad 
\bar b_{L+1} := \bar c \in \R^{w_{L+1}},
\end{gathered}
\end{equation*}
where $1_p \in \R^p$, $p \in \mathbb{N}$, again denotes the vector whose entries are all one. 
For this $\bar \alpha$, it clearly holds $\Psi(\bar \alpha) = z_{\bar c} \in C(K)$ and 
\begin{equation*}
\bar A_{j}\left (
\varphi_{j-1}^{\bar A_{j-1}, \bar b_{j-1}} \circ ... \circ \varphi_{1}^{\bar A_{1}, \bar b_{1}}(x)
\right )
+
\bar b_j
=
c_j 1_{w_j} \in 
 \left (c_j - \ \varepsilon_j,  c_j + \varepsilon_j \right )^{w_j}
\end{equation*}
for all $x \in K$. 
Let us denote the collection of all parameters of $\psi$ belonging to the degrees of freedom
 $A_{L+1},...,A_1$ and $b_L,...,b_1$ with $\alpha'$
and the space of all such $\alpha'$ with $D'$ (again endowed with the Euclidean norm
of the associated space $\R^{m - 1}$
analogously to the proof of \cref{lem:aff1}). 
Then the compactness of $K$, the openness of $I_j$,
 the fact that $\sigma_j$ is constant on $I_j$,
 the definition of $\bar \alpha$,
the architecture of $\psi$,
and the continuity of the function 
$\smash{D \times \R^d  \ni (\alpha, x) \mapsto A_j (\varphi_{j-1}^{A_{j-1}, b_{j-1}} \circ ... \circ \varphi_{1}^{A_{1},b_{1}}(x)) + b_j \in \R^{w_j}}$ imply that there exist $r > 0$ 
and a continuous $\Phi\colon D' \to \R$ such that $\Phi(\bar \alpha') = 0$ holds and 
\begin{equation}
\label{eq:randomeq2636}
\psi(\alpha, x) = \Phi(\alpha') + b_{L+1},
\qquad \forall x \in K,\qquad \forall \alpha \in B_r^D(\bar \alpha).
\end{equation}
Define $g\colon D' \to \R$, $g(\alpha') := \bar c - \Phi(\alpha')$. 
Then $g$ is continuous, it holds $g(\bar \alpha') = \bar b_{L+1}$, and we can find a number $r' > 0$
such that
\[
E := 
\left \{
\alpha \in D
\; \Big |\;
\alpha' \in B_{r'}^{D'}(\bar \alpha'),
b_{L+1} = g(\alpha')
\right \} \subset B_r^D(\bar \alpha).
\]
For all $\tilde \alpha \in E$, it now follows from \eqref{eq:randomeq2636} and the definition of $g$ that 
\[
\psi(\tilde \alpha, x)  
= \Phi(\tilde \alpha') + \tilde b_{L+1} 
= \Phi(\tilde \alpha') + g(\tilde \alpha')
= \bar c,\qquad \forall x \in K.
\]
Due to the properties of $\bar c$ and the definition of $r$, this yields 
\[
\LL(\Psi(\tilde \alpha), y_T) = 
\LL(z_{\bar c}, y_T) =
\min_{c \in \R} \LL(z_{c}, y_T)
=
\min_{\alpha \in B_r^D(\bar \alpha)} \LL(\Psi(\alpha), y_T) 
\]
for all $\tilde \alpha \in E \subset  B_r^D(\bar \alpha)$. Thus, 
all elements of $E$ are local minima of \eqref{eq:P} satisfying \eqref{eq:randomeq26353-3}.
That $E$ has Hausdorff dimension at least $\dim(D) - 1$ follows completely 
analogously to the proof of \cref{lem:aff1}. 
\end{proof}

As already mentioned in the introduction, 
the approach
that we have used in \cref{lem:aff1,lem:aff2}
to construct the local minima in $E$ is not new. 
The idea to choose biases and weights such that 
the network inputs only come into contact with the affine linear
parts of the activation functions $\sigma_i$
can also be found in various other contributions, e.g., \cite{Christof2021,Ding2020,Goldblum2020Truth,He2020,Yun2018}. 
The main challenge in the context of \cref{th:spuriousminima-nonconstant,th:spuriousminima-constant} 
is proving that  
the local minima in \cref{lem:aff1,lem:aff2} are indeed spurious for 
generic $y_T$ and arbitrary $\sigma_i$, $L$, $w_i$, and $\LL$.
The following two lemmas show that this spuriousness can be 
established without  lengthy computations and manual constructions
by means of \cref{th:nohalfspaces}.
\begin{lemma}
\label{lem:aff3}
Suppose that $K$, $\psi$, $w_i$, $L$, $\sigma_i$, $y_T$, and $\LL$
satisfy the assumptions of \cref{th:spuriousminima-nonconstant}
and let $z_{a,c} \in C(K)$ be defined as in \cref{lem:aff1}.
Then it holds 
\begin{equation}
\label{eq:ranomd263536}
\inf_{\alpha \in D} \LL(\Psi(\alpha), y_T)
<
 \min_{(a,c) \in \R^d \times \R} \LL(z_{a,c}, y_T).
\end{equation}
\end{lemma}
\begin{proof}
We argue by contradiction.
Suppose that the assumptions of \cref{th:spuriousminima-nonconstant} are satisfied and that 
\eqref{eq:ranomd263536} is false. Then it holds
\begin{equation}
\label{eq:randomeq3636-7656}
\LL(\Psi(\alpha), y_T)
\geq 
 \min_{(a,c) \in \R^d \times \R} \LL(z_{a,c}, y_T)
 =
 \LL(z_{\bar a, \bar c}, y_T),\qquad \forall \alpha \in D,
\end{equation}
where $(\bar a, \bar c) \in \R^d \times \R$ is the minimizer from assumption  \ref{th:spur1:item:v}  
of \cref{th:spuriousminima-nonconstant}.
To see that this inequality cannot be true, we
consider network parameters $\alpha = \{ (A_{i}, b_{i})\}_{i=1}^{L+1} \in D$ of the form 
\begin{equation}
\label{eq:randomeq3636468}
\begin{gathered}
A_1 := 
\begin{pmatrix}
\bar A_1
\\
\tilde A_1
\end{pmatrix}
\in \R^{w_1 \times w_0},
\\
A_i := 
\begin{pmatrix}
\bar A_i & 0_{1 \times (w_{i-1} - 1)}
\\
0_{(w_i - 1) \times 1}& \tilde A_i
\end{pmatrix}
\in \R^{w_{i} \times w_{i-1}},
~i=2,...,L,
\\
A_{L+1} := (\bar A_{L+1}~~ \tilde A_{L+1} ) \in \R^{w_{L+1} \times w_L},
\\
b_i :=
\begin{pmatrix}
\bar b_i
\\
\tilde b_i
\end{pmatrix}
\in \R^{w_i}
,
~
i=1,...,L,
\qquad
b_{L+1}
:=
\bar b_{L+1}  + \tilde b_{L+1} \in \R
\end{gathered}
\end{equation}
with arbitrary but fixed $\bar A_1\in \R^{1 \times d}$, $\bar A_i \in \R$, $i=2,...,L+1$, $\bar b_i \in \R$, $i=1,...,L+1$,
$\tilde A_1 \in \R^{(w_1 - 1) \times d}$, $\tilde A_i \in \R^{(w_i - 1) \times (w_{i-1}- 1)}$, $i=2,...,L$,
$\tilde A_{L+1} \in \R^{1 \times (w_L - 1)}$,  
$\tilde b_i \in \R^{w_i - 1}$, $i=1,...,L$, and $\tilde b_{L+1} \in \R$. 
Here, $0_{p\times q} \in \R^{p \times q}$ again denotes  a zero matrix. 
Note that such a structure of the network parameters is possible due to the assumption $w_i \geq 2$, $i=1,...,L$, in \ref{th:spur1:item:i}.
Using \eqref{eq:network}, it is easy to check that every $\alpha$ of the type \eqref{eq:randomeq3636468}
satisfies $\psi(\alpha, x) = \bar \psi(\bar \alpha, x) + \tilde \psi(\tilde \alpha, x)$ for all $x \in \R^d$, 
where $\bar \psi$ is a neural network as in \eqref{eq:network} with
depth $\bar L = L$, widths $\bar w_i = 1$, $i=1,...,L$,  activation functions $\sigma_i$,
and network parameter $\bar \alpha = \{ (\bar A_{i}, \bar b_{i})\}_{i=1}^{L+1}$
and where 
$\tilde\psi$ is a neural network as in \eqref{eq:network} with
depth $\tilde L = L$, widths $\tilde w_i = w_i - 1$, $i=1,...,L$, activation functions $\sigma_i$,
and network parameter $\tilde \alpha = \{ (\tilde A_{i}, \tilde b_{i})\}_{i=1}^{L+1}$. 
In combination with \eqref{eq:randomeq3636-7656}, this implies 
\begin{equation}
\label{eq:randomeq3263ge63h}
 \LL(\bar \Psi(\bar \alpha) + \tilde \Psi(\tilde \alpha), y_T)
\geq 
 \LL(z_{\bar a, \bar c}, y_T),\qquad \forall \bar \alpha \in \bar D,\qquad \forall \tilde \alpha \in \tilde D. 
\end{equation}
Here, we have used the symbols $\bar D$ and $\tilde D$ to denote the 
parameter spaces of $\bar \psi$ and $\tilde \psi$, respectively, 
and the symbols $\bar \Psi$ and $\tilde \Psi$ to denote the functions into $C(K)$
associated with $\bar \psi$ and $\tilde \psi$ defined in \eqref{eq:Psidef}. 
Note that, by exactly the same arguments as in the proof of \cref{lem:aff1},
we obtain that there exists $\bar \alpha \in \bar D$ with $\bar \Psi(\bar \alpha) = z_{\bar a, \bar c}$.
Due to \eqref{eq:randomeq3263ge63h} and the fact that $\tilde A_{L+1}$ and $\tilde b_{L+1}$ can be rescaled at will,
this yields
 \begin{equation}
 \label{eq:randomeq36736g4674hd7e3}
\LL(z_{\bar a, \bar c} + s \tilde \Psi(\tilde \alpha), y_T)
\geq 
 \LL(z_{\bar a, \bar c}, y_T),\qquad \forall \tilde \alpha \in \tilde D,\qquad \forall s \in (0, \infty). 
\end{equation}
Since $z_{\bar a, \bar c} \neq y_T$ holds by \ref{th:spur1:item:iii} and 
since $\LL$ 
 is Gâteaux differentiable in its first argument 
 with a nonzero 
derivative $\partial_1 \LL(v, y_T)$ at all points $(v, y_T) \in C(K) \times C(K)$ 
satisfying $v \neq y_T$ by \ref{th:spur1:item:iv},
we can rearrange  \eqref{eq:randomeq36736g4674hd7e3}, divide by $s>0$, and pass to the limit $s \to 0^+$
(for an arbitrary but fixed $\tilde \alpha \in \tilde D$) to obtain 
 \begin{equation}
 \label{eq:randomeq26346545}
\left \langle 
\partial_1 \LL(z_{\bar a, \bar c}, y_T), \tilde \Psi(\tilde \alpha) 
\right \rangle_{C(K)}
\geq 
0,
\qquad \forall \tilde \alpha \in \tilde D,
\end{equation}
with a measure $\partial_1 \LL(z_{\bar a, \bar c}, y_T) \in \MM(K) \setminus \{0\}$.
From \cref{th:nohalfspaces},
we know that \eqref{eq:randomeq26346545} can only be true if $\partial_1 \LL(z_{\bar a, \bar c}, y_T) = 0$ holds.
Thus, we arrive at a contradiction, \eqref{eq:randomeq3636-7656} 
cannot be correct, and the proof is complete.
\end{proof}

\begin{lemma}
\label{lem:aff4}
Suppose that $K$, $\psi$, $w_i$, $L$, $\sigma_i$, $y_T$, and $\LL$
satisfy the assumptions of \cref{th:spuriousminima-constant}
and let $z_{c} \in C(K)$ be defined as in \cref{lem:aff2}.
Then it holds 
\begin{equation}
\label{eq:randomeq26736}
\inf_{\alpha \in D} \LL(\Psi(\alpha), y_T)
<
 \min_{c \in  \R} \LL(z_{c}, y_T).
\end{equation}
\end{lemma}
\begin{proof}
The proof of \cref{lem:aff4} is analogous to that of \cref{lem:aff3} but simpler. 
Suppose that \eqref{eq:randomeq26736} is false and that the assumptions of \cref{th:spuriousminima-constant} 
are satisfied. 
Then it holds
\begin{equation}
\label{eq:randomeq485885th8}
\LL(\Psi(\alpha), y_T)
\geq 
 \min_{c \in \R} \LL(z_{c}, y_T)
 =
 \LL(z_{\bar c}, y_T),\qquad \forall \alpha \in D,
\end{equation}
where $\bar c \in \R$ denotes the minimizer from
point \ref{th:spur2:item:iv} of \cref{th:spuriousminima-constant}.
By exploiting that the parameter  $\alpha = \{ (A_{i}, b_{i})\}_{i=1}^{L+1} \in D$
is arbitrary, by shifting the bias $b_{L+1}$ by $\bar c$, and by subsequently scaling 
$A_{L+1}$ and $b_{L+1}$ in \eqref{eq:randomeq485885th8}, we obtain that
\begin{equation}
\label{eq:randomeq26356346}
\LL(z_{\bar c} + s\Psi(\alpha), y_T)
\geq 
 \LL(z_{\bar c}, y_T),\qquad \forall \alpha \in D,\qquad \forall s \in (0, \infty).
\end{equation}
In combination with assumptions \ref{th:spur2:item:ii} and \ref{th:spur2:item:iii}
of \cref{th:spuriousminima-constant},
\eqref{eq:randomeq26356346} yields -- completely analogously to \eqref{eq:randomeq26346545} -- that 
there exists a measure $\partial_1 \LL(z_{\bar c}, y_T) \in \MM(K) \setminus \{0\}$ satisfying 
 \begin{equation*}
\left \langle 
\partial_1 \LL(z_{\bar c}, y_T),  \Psi( \alpha) 
\right \rangle_{C(K)}
\geq 
0,
\qquad \forall  \alpha \in D.
\end{equation*} 
By invoking \cref{th:nohalfspaces}, we now again arrive at a contradiction. 
Thus, \eqref{eq:randomeq485885th8} cannot be true and the assertion of the lemma follows. 
\end{proof}

To establish \cref{th:spuriousminima-nonconstant,th:spuriousminima-constant},
it suffices to combine \cref{lem:aff1,lem:aff3} and \cref{lem:aff2,lem:aff4},
respectively. This completes the proof of our main results on the existence of spurious local 
minima in training problems of the type \eqref{eq:P}.

\begin{remark}~
\begin{itemize}
\item As it is irrelevant for our analysis whether the function 
$ \alpha \mapsto \LL(\Psi(\alpha), y_T)$ appearing in (P) is used 
to train a network or to validate the generalization properties of a trained network, 
Theorems 3.1 and 3.2 also establish the existence of spurious local minima for the generalization error.

\item We expect that it is possible to extend \cref{th:spuriousminima-nonconstant,th:spuriousminima-constant}
to training problems defined on the whole of $\R^d$ provided the activation functions $\sigma_i$
and the loss function $\LL$ are sufficiently well-behaved. We omit a detailed discussion of this extension here 
since it requires nontrivial modifications of the functional analytic setting and leave this topic
for future research.

\item The distinction between the ``nonconstant affine segment''-case and the ``constant segment''-case
in \cref{th:spuriousminima-nonconstant,th:spuriousminima-constant} and 
\cref{lem:aff1,lem:aff3,lem:aff2,lem:aff4} is necessary. This can be seen, e.g., in the formulas in \eqref{eq:baralpha}
which degenerate when one of the slopes $\beta_i$ is equal to zero.

\item If a network of the type \eqref{eq:network} with an additional activation function $\sigma_{L+1}$ 
acting on the last layer is considered, then one can simply include $\sigma_{L+1}$ into the loss function 
by defining $\tilde \LL(v, y_T) := \LL(\sigma_{L+1} \circ v, y_T)$. Along these lines, 
our results can be applied 
to networks with a nonaffine last layer as well
(provided the function $\tilde \LL$ still satisfies the assumptions of \cref{th:spuriousminima-nonconstant,th:spuriousminima-constant}).
\end{itemize}
\end{remark}

\section{Further consequences of the nonexistence of supporting half-spaces}
\label{sec:5}
The aim of this section is to point out some further consequences 
of \cref{th:nohalfspaces}. Our main focus will be on the implications that 
this theorem has for the well-posedness properties of 
best approximation problems for neural networks
in function space. We begin by noting that the nonexistence of 
supporting half-spaces for the image $\Psi(D)$ in \eqref{eq:randomeq2535} 
implies that the closure of $\Psi(D)$ can only be convex if it is equal to 
the whole of $C(K)$. More precisely, we have the following result: 

\begin{corollary}
\label{cor:nonvonvex}
Let $K \subset \R^d$, $d \in \mathbb{N}$, be a nonempty and compact set and let 
$\psi\colon D \times \R^d \to \R$ be a neural network as in \eqref{eq:network}
with depth $L \in \mathbb{N}$, widths $w_i \in \mathbb{N}$, and continuous 
nonpolynomial activation functions $\sigma_i$. Suppose that $(Z, \|\cdot\|_Z)$
is a real normed space and that $\iota\colon C(K) \to Z$ is a linear and continuous map 
with a dense image. Then 
the set $\closure_Z(\iota(\Psi(D)))$ is either nonconvex or equal to $Z$.
\end{corollary}
\begin{proof}
Assume that  $\closure_Z(\iota(\Psi(D)))$  is convex and that $\closure_Z(\iota(\Psi(D))) \neq Z$.
Then there exists $z \in Z \setminus \closure_Z(\iota(\Psi(D)))$ and it follows from the 
separation theorem for convex sets in normed spaces \cite[Corollary I-1.2]{Ekeland1976}
that we can find $\nu \in Z^* \setminus \{0\}$ and $c \in \R$ such that 
\begin{equation}
\label{eq:randomeq26354}
\left \langle \nu, \iota(\Psi(\alpha))\right \rangle_Z = \left \langle \iota^*(\nu), \Psi(\alpha)\right \rangle_{C(K)} \leq c,
\qquad \forall \alpha \in D.
\end{equation}
Here, $Z^*$ denotes the topological dual of $Z$, $\left \langle \cdot,\cdot\right\rangle_Z\colon Z^* \times Z \to \R$
denotes the dual pairing in $Z$, and $\iota^*\colon Z^* \to C(K)^* = \MM(K)$ denotes the adjoint of $\iota$
as defined in \cite[section 9]{Clason2020}. Due to \cref{th:nohalfspaces}, 
\eqref{eq:randomeq26354} is only possible if $\iota^*(\nu) = 0$, i.e., if
\[
\left \langle \iota^*(\nu), v\right \rangle_{C(K)}
=
\left \langle \nu, \iota(v)\right \rangle_Z = 0, \qquad \forall v \in C(K).
\]
As $\iota(C(K))$ is dense in $Z$, this yields $\nu = 0$ which is a contradiction. 
Thus, the set $\closure_Z(\iota(\Psi(D)))$ is either nonconvex or equal to $Z$ and the proof is complete. 
\end{proof}

We remark that, for activation functions possessing a point of differentiability with a nonzero derivative, 
a version of \cref{cor:nonvonvex} has already been proven in \cite[Lemma C.9]{Petersen2021}.
By using \cref{th:nohalfspaces} and the separation theorem, we can avoid the assumption
that such a point of differentiability exists and obtain 
\cref{cor:nonvonvex} for all nonpolynomial continuous activations $\sigma_i$.
In combination with classical results on the properties of Chebychev sets, see \cite{Vlasov1970}, the nonconvexity 
of the set $\closure_Z(\iota(\Psi(D)))$ in \cref{cor:nonvonvex} immediately implies that 
the problem of determining a best approximating element for a given $u \in Z$
from the set $\closure_Z(\iota(\Psi(D)))$ of all elements of $Z$ that can be approximated 
by points of the form $\iota(\Psi(\alpha))$ is always ill-posed in the sense of Hadamard if 
$Z$ is a strictly convex Banach space with a strictly convex dual and $\iota(\Psi(D))$ is not dense. 

\begin{corollary}
\label{cor:illposed}
Let $K$, $\psi$, $L$, $w_i$, $\sigma_i$, $(Z, \|\cdot\|_Z)$,
and $\iota$ be as in \cref{cor:nonvonvex}. Assume additionally that $(Z, \|\cdot\|_Z)$ 
is a Banach space and that  $(Z, \|\cdot\|_Z)$ 
and its topological dual $(Z^*, \|\cdot\|_{Z^*})$ are strictly convex. 
Define $\Pi$ to be the best approximation map associated with the set $\closure_Z(\iota(\Psi(D)))$, i.e., 
the set-valued projection operator
\begin{equation}
\label{eq:bestapprox}
\Pi\colon Z \rightrightarrows Z,
\qquad u \mapsto \argmin_{z \in  \closure_Z(\iota(\Psi(D)))} \left \| u - z\right \|_Z.
\end{equation}
 Then exactly one of the following is true:
\begin{enumerate}[label=\roman*)]
\item $\closure_Z(\iota(\Psi(D)))$ is equal to $Z$ and $\Pi$ is the identity map.
\item\label{illposed:item:ii} There does not exist a function $\pi\colon Z \to Z$
such that $\pi(z) \in \Pi(z)$ holds for all $z \in Z$ and such that $\pi$ is continuous in an open neighborhood of the origin. 
\end{enumerate}
\end{corollary}
\begin{proof}
This immediately follows from \cref{cor:nonvonvex}, \cite[Theorem 3.5]{Kainen1999},
and the fact that the set $\closure_Z(\iota(\Psi(D)))$ is a cone.
\end{proof}

Note that there are two possible reasons for 
the nonexistence of a selection $\pi$ with the properties in point \ref{illposed:item:ii} 
of \cref{cor:illposed}. The first one is that there exists an element $u \in Z$ for which the set $\Pi(u)$ is empty, i.e., for which
the best approximation problem associated with the right-hand side of \eqref{eq:bestapprox}  
does not possess a solution.
The second one is that $\Pi(u) \neq \emptyset$ holds for all $u \in Z$ but that every selection $\pi$
taken from $\Pi$ is discontinuous at some point $u$, i.e., that
there exists an element $u \in Z$ for which the solution set of 
the best approximation problem associated with the right-hand side of \eqref{eq:bestapprox} is unstable w.r.t.\ 
small perturbations of the problem data.
In both of these cases, one of the conditions for Hadamard well-posedness is violated, 
see \cite[section 2.1]{Kabanikhin2012}, so that 
\Cref{cor:illposed} indeed implies that the problem of determining best 
approximations is ill-posed 
when  $\closure_Z(\iota(\Psi(D))) \neq Z$ holds.

To make \cref{cor:illposed} more tangible, we state its consequences for best approximation problems 
posed in reflexive Lebesgue spaces, cf.\ \cref{lem:LptrackingFunctional}. 
Such problems arise when $Z$ is equal to $L^p_\mu(K)$ for some $\mu \in \MM_+(K)$ and $p \in (1,\infty)$
and when $\iota\colon C(K) \to L^p_\mu(K)$ is the inclusion map.  
In the statement of the next corollary, we drop the inclusion map $\iota$ in the notation for the sake of readability. 

\begin{corollary}
\label{cor:LPillposed}
Suppose that $K \subset \R^d$, $d \in \mathbb{N}$, is a nonempty compact set and that
$\psi$ is a neural network as in \eqref{eq:network}
with depth $L \in \mathbb{N}$, widths $w_i \in \mathbb{N}$, and continuous 
nonpolynomial activation functions $\sigma_i$. Assume that $\mu \in \MM_+(K)$
and $p \in (1, \infty)$ are given and that the image $\Psi(D)$ of the map $\Psi\colon D \to C(K)$
in \eqref{eq:Psidef} is not dense in $L^p_\mu(K)$. Then there does not 
exist a function $\pi\colon L^p_\mu(K) \to L^p_\mu(K)$ such that 
\begin{equation}
\label{eq:Lpproj}
\pi(u) \in \argmin_{z \in  \closure_{L^p_\mu(K)}(\Psi(D))} \left \| u - z\right \|_{L^p_\mu(K)},\qquad \forall u \in L^p_\mu(K),
\end{equation}
holds and such that $\pi$ is continuous in an open neighborhood of the origin.
\end{corollary}
\begin{proof}
From \cite[Example 1.10.2, Theorem 5.2.11]{Megginson1998},
it follows that $L^p_\mu(K)$ is uniformly convex with a uniformly convex dual,
and from \cite[Proposition 7.9]{Folland1999}, 
we obtain that the inclusion map $\iota\colon C(K) \to L^p_\mu(K)$ is linear and continuous with a dense image. 
The claim thus immediately follows from \cref{cor:illposed}. 
\end{proof}

As already mentioned in \cref{sec:1}, for neural networks with a single hidden layer,
a variant of \cref{cor:LPillposed} has also been proven in \cite[section 4]{Kainen1999}.
For related results, see also \cite{Kainen2001,Petersen2021}.
We obtain the discontinuity of $L^p_\mu(K)$-best approximation 
operators for networks of arbitrary depth here as a
consequence of  \cref{th:nohalfspaces} 
and thus, at the end of the day, as a corollary of the universal approximation theorem.
This again highlights the connections that exist between the approximation capabilities 
of neural networks and the landscape/well-posedness properties of 
the optimization problems that have to be solved in order to determine 
neural network best approximations.

We remark that,
to get an intuition for the 
geometric properties of the image $\Psi(D) \subset C(K)$ that are responsible
for the effects in \cref{th:spuriousminima-nonconstant}, \cref{th:spuriousminima-constant}, and \cref{cor:LPillposed}, 
one can indeed plot this set in simple situations. 
Consider, for example, the case $d=1$, $K=\{-1,0,2\}$, $\mu = \delta_{-1} + \delta_0 + \delta_{2}$,
$L = 1$, $w_1=1$, and $p=2$, where $\delta_x$ again denotes a Dirac measure supported at $x \in \R$.  
For these $K$ and $\mu$, we have $C(K) \cong  L^2_\mu(K) \cong \R^3$ and 
the image $\Psi(D) \subset C(K)$ of the map $\Psi$ in \eqref{eq:Psidef} can be 
identified with a subset of $\R^3$, namely,
\[
\Psi(D) =
\left \{
z \in \R^3
\;
\Big |
\;
z = \big ( \psi(\alpha,-1), \psi(\alpha,0), \psi(\alpha,2) \big )^\top \text{for some } \alpha \in D
\right \}. 
\]
Further, the best approximation problem associated with the right-hand side of \eqref{eq:Lpproj} 
simply becomes the problem of determining the set-valued Euclidean projection of 
a point $u \in \R^3$ onto $\closure_{\R^3}(\Psi(D))$, i.e., 
\begin{equation}
\label{eq:randomeq27464eh}
\text{Minimize} \quad  \left |  u - z \right |\qquad \text{w.r.t.}\quad z \in \closure_{\R^3}(\Psi(D)).
\end{equation}
This makes it possible to visualize the image $\Psi(D)$ and to interpret the $L^2_\mu(K)$-best
approximation operator associated with $\psi$ geometrically. 
The sets $\Psi(D)$ that are obtained in the above situation 
for the ReLU-activation $\sigma_{\textup{relu}}(s) := \max(0, s)$
and the SQNL-activation 
\[
\sigma_{\textup{sqnl}}(s) :=
\begin{cases}
-1 & \text{ if } s \leq -2
\\
s + s^2 /4& \text{ if } -2 < s \leq 0
\\
s - s^2 /4& \text{ if } 0 < s \leq 2
\\
1 & \text{ if } s > 2
\end{cases}
\]
can be seen in \cref{fig:graphplots}. Note that, since both of these functions
are monotonically increasing, the assumption $L = w_1 = 1$ and the architecture in \eqref{eq:network}
imply that $(0, 1, 0)^\top \not\in \closure_{\R^3}(\Psi(D))$ holds.
This shows that, for both the ReLU- and the SQNL-activation,
the resulting network falls under the scope of \cref{cor:LPillposed}.
Since  $\sigma_{\textup{relu}}$ and $\sigma_{\textup{sqnl}}$ possess constant segments, 
the training problems
\begin{equation}
\label{eq:randomeq27464eh-2}
\text{Minimize} \quad  \left |  u - \Psi(\alpha) \right |\qquad \text{w.r.t.}\quad \alpha \in D
\end{equation}
associated with 
these activation functions are moreover covered by \cref{th:spuriousminima-constant},
cf.\  \cref{lem:LptrackingFunctional}.
As \cref{fig:graphplots} shows, the sets $\Psi(D)$ obtained for 
$\sigma_{\textup{relu}}$ and $\sigma_{\textup{sqnl}}$ along the above lines are highly nonconvex and 
locally resemble two-dimensional subspaces of $\R^3$ at many points. 
Because of these properties, it is only natural 
that the resulting $L^2_\mu(K)$-best approximation operators, 
i.e., the Euclidean projections onto $ \closure_{\R^3}(\Psi(D))$, 
possess discontinuities and give rise to training problems that contain various spurious local minima. 
We remark that the examples in \cref{fig:graphplots} improve a construction 
in \cite[section 4]{Christof2021}, where a similar visualization for a 
more academic network was considered. We are able to overcome the restrictions of \cite{Christof2021} 
here due to \cref{th:pinkus,th:nohalfspaces}. Note that depicting the image $\Psi(D)$ of 
a neural network along the lines of \cref{fig:graphplots} only works well for very small architectures.
Larger networks are too expressive to be properly visualized in three dimensions.

\begin{figure}[H]
\centering
\begin{subfigure}{.5\textwidth}
  \centering
  \includegraphics[width= .99\linewidth]{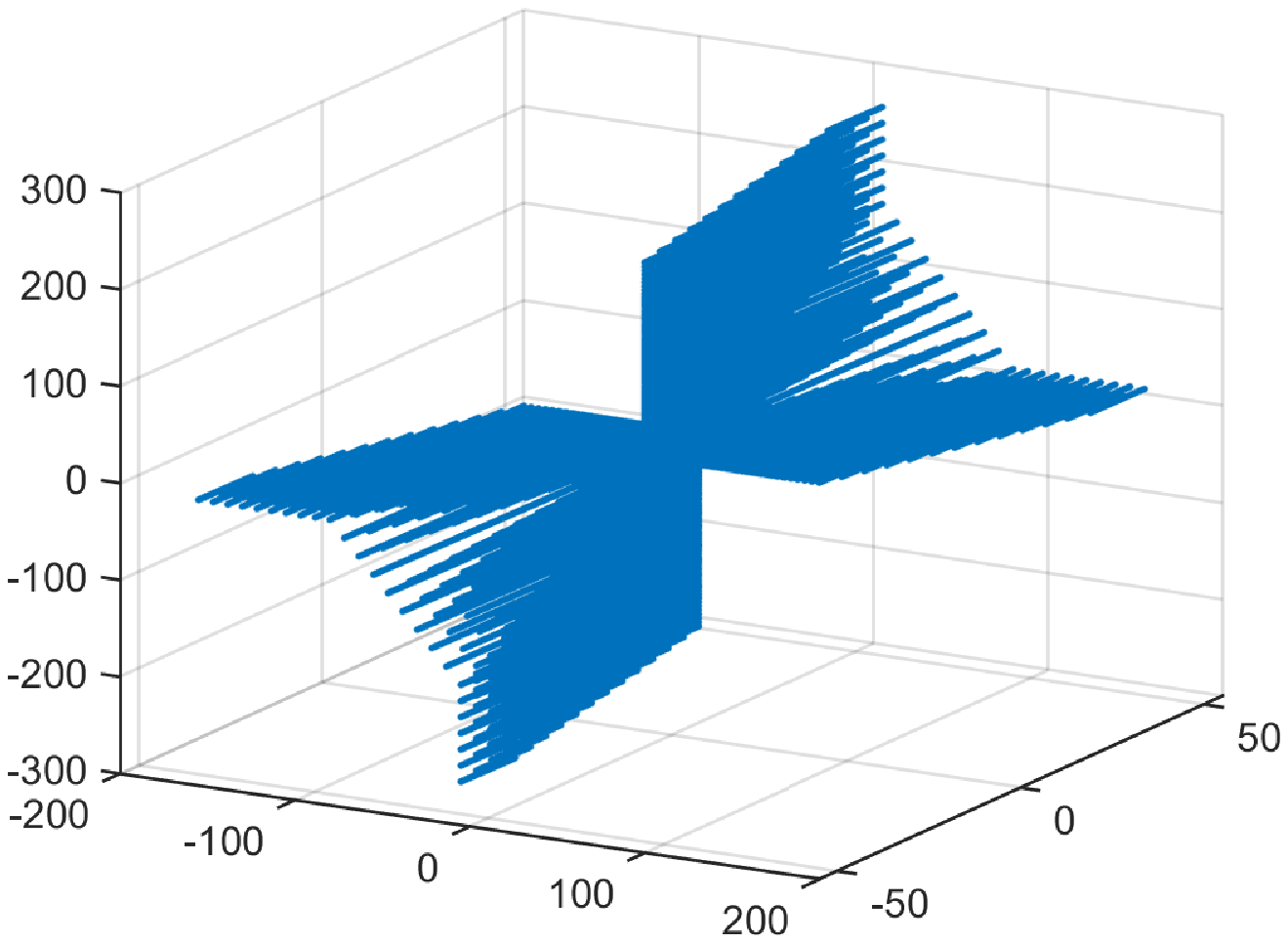}
  \caption{ $\Psi(D)$ for $\sigma_{\textup{relu}}$}
  \label{fig:sub1}
\end{subfigure}%
\begin{subfigure}{.5\textwidth}
  \centering
  \includegraphics[width=.99\linewidth]{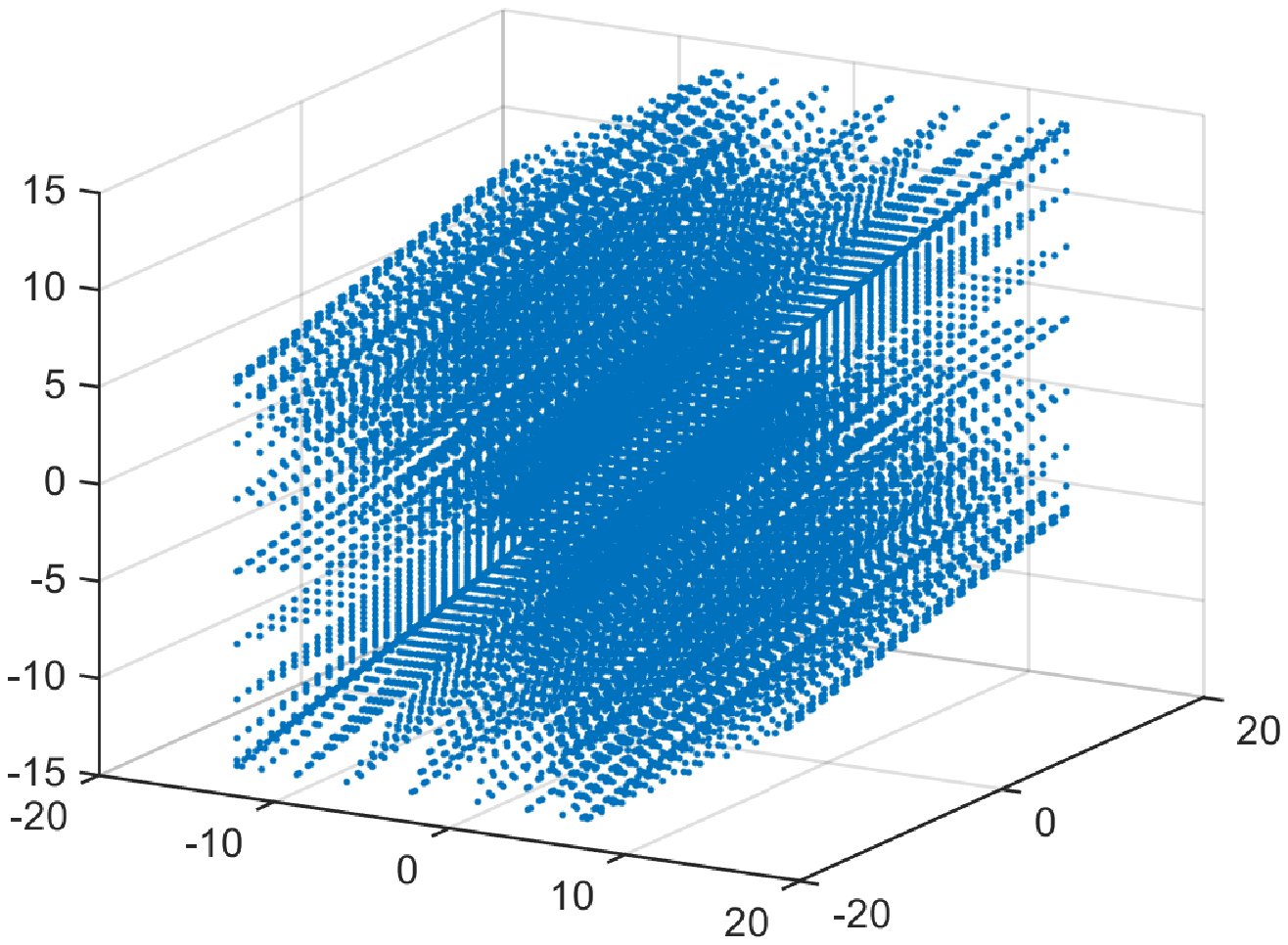}
  \caption{$\Psi(D)$ for $\sigma_{\textup{sqnl}}$}
  \label{fig:sub2}
\end{subfigure}
\vspace{-0.2cm}
\caption{Scatter plot of the image $\Psi(D)$
of the function $\Psi\colon D \to C(K) \cong  L^2_\mu(K) \cong \R^3$ in the case 
$d=1$, $K=\{-1,0,2\}$, $\mu = \delta_{-1} + \delta_0 + \delta_{2}$,
$L = 1$, and $w_1=1$ for the ReLU- and the SQNL-activation function.
For the weights $A_1, A_2 \in \R$, we used samples from the interval $[-10,10]$,
and for the biases $b_1, b_2 \in \R$, from the interval $[-5,5]$. Solving the problems
\eqref{eq:randomeq27464eh} or \eqref{eq:randomeq27464eh-2} for a given $u$ 
corresponds to calculating the set-valued Euclidean projection of $u$ onto these sets. 
}
\label{fig:graphplots}
\end{figure}
\vspace{-0.5cm}

We would like to point out that the ``space-filling'' cases 
$\closure_Z(\iota(\Psi(D))) = Z$ and 
 $\closure_{L^p_\mu(K)}(\Psi(D)) = L^p_\mu(K)$
 in \cref{cor:nonvonvex,cor:illposed,cor:LPillposed} are not as 
 pathological as one might think at first glance. 
 In fact, in many applications, neural networks are 
 trained in an ``overparameterized'' regime in which
 the number of degrees of freedom in $\psi$
exceeds the number of training samples by far and 
in which $\psi$ is able to fit arbitrary training data with zero error, see 
\cite{Zhu2019,Chen2021,Cooper2020,Li2018,Oymak2019,Safran2021}.
In the situation of \cref{lem:LptrackingFunctional}, this means that a measure $\mu$
of the form $\mu =\frac{1}{n} \sum_{k=1}^n  \delta_{x_k}$ supported on a finite set  
$K = \{x_1,...,x_n\} \subset \R^d$, $n \in \mathbb{N}$, is considered 
which satisfies $n \ll m = \dim(D)$. The absence of the ill-posedness effects 
in \cref{cor:LPillposed} is a possible explanation for the observation that 
overparameterized neural networks are far easier to train than their non-overparameterized counterparts,
cf.\ \cite{Zhu2019,Li2018,Oymak2019,Safran2021}.
We remark that, for  overparameterized finite-dimensional training problems, 
numerically determining a minimizer is also simplified by the fact that, 
in the case $m = \dim(D) \gg n$, the
Jacobian of the map $D \ni \alpha \mapsto \Psi(\alpha) \in C(K) \cong \R^n$ 
typically has full rank on a large subset of the parameter space $D$, cf.\ \cite{DingWidth2022,Nguyen2018}. Note that, 
for infinite $K$, there is no analogue to this effect since the
Gâteaux derivative of the map $D \ni \alpha \mapsto \Psi(\alpha) \in C(K)$ can never be surjective if 
$C(K)$ is infinite-dimensional. \Cref{th:spuriousminima-constant,th:spuriousminima-nonconstant}
further show that the Jacobian of the mapping $D \ni \alpha \mapsto \Psi(\alpha) \in C(K) \cong \R^n$ 
can indeed only be expected to have full rank on a large set (e.g., a.e.)
when an  overparameterized finite-dimensional training problem is considered, 
but not everywhere.

Even though there is no sensible notion of ``overparameterization''
in the infinite- dimensional setting, 
it is still possible for a neural network to satisfy the conditions 
$\closure_Z(\iota(\Psi(D))) = Z$ and 
 $\closure_{L^p_\mu(K)}(\Psi(D)) = L^p_\mu(K)$ in \cref{cor:nonvonvex,cor:illposed,cor:LPillposed} for a non-finite training set $K$. 
 In fact,  in the case $d=1$, it can be shown that the set of activation 
functions that give rise to a ``space-filling'' network is dense in $C(\R)$ 
in the topology of uniform convergence on compacta. There thus indeed exist 
\emph{many} choices of activation functions $\sigma\colon \R \to \R$ for which the density conditions 
$\closure_Z(\iota(\Psi(D))) = Z$ and 
 $\closure_{L^p_\mu(K)}(\Psi(D)) = L^p_\mu(K)$  in \cref{cor:nonvonvex,cor:illposed,cor:LPillposed}
 hold for arbitrary spaces $Z$ and arbitrary measures $\mu$. To be more precise, we have:\pagebreak
 
 \begin{lemma}
 \label{lem:dense}
 Consider a nonempty compact set  $K \subset \R$ and a neural network 
$\psi$ as in \eqref{eq:network} with depth $L \in \mathbb{N}$, widths $w_i \in \mathbb{N}$, and $d = 1$. 
Suppose that $\sigma_i = \sigma$ holds for all $i=1,...,L$ 
with a function $\sigma \in C(\R)$. 
Then, for all $\varepsilon > 0$ and all nonempty open intervals $I \subset \R$, there exists a
function $\tilde \sigma \in C(\R)$ such that $\sigma \equiv \tilde \sigma$ holds in $\R \setminus I$,
such that $|\sigma(s) - \tilde \sigma(s)| < \varepsilon$ holds for all $s \in \R$,
and such that the neural network $\tilde \psi$ obtained by replacing $\sigma$ with $\tilde \sigma$ in $\psi$
satisfies $\closure_{C(K)}(\tilde \Psi(D)) = C(K)$. 
 \end{lemma}
 
 \begin{proof}
The lemma is an easy consequence of the separability of $(C(K), \|\cdot \|_{C(K)})$,
cf.\ \cite{Maiorov1999}.
Since we can replace $K$ by a closed bounded interval that contains $K$ to prove the claim,
since we can rescale and translate the argument $x$ of $\psi$ by means of $A_1$ and $b_1$, 
and since we can again consider parameters 
of the form \eqref{eq:randomeq364564ge}, we may assume w.l.o.g.\ that 
$K = [0,1]$ holds and that all layers of $\psi$ have width one.
Suppose that a number $\varepsilon > 0$ and a nonempty open interval $I$ are given. 
Using the continuity of $\sigma$, it is easy to check that there exists a
function $\bar \sigma \in C(\R)$ that satisfies $\sigma \equiv \bar \sigma$ in $\R \setminus I$,
$|\sigma(s) - \bar \sigma(s)| < \varepsilon/2$ for all $s \in \R$, 
and $\bar \sigma = \mathrm{const}$ in $(a, a + \eta)$
for some $a \in \R$ and $\eta > 0$ with $(a, a + \eta) \subset I$. 
Let $\{p_k\}_{k=1}^\infty \subset C([0,1])$ denote the countable collection 
of all polynomials on $[0,1]$ that have rational coefficients and that are not identical zero,
starting with $p_1(x) = x$,
and let $\phi\colon \R \to \R$ be the unique element of $C(\R)$
with the following properties:
\begin{enumerate}[label=\roman*)]
\item $\phi \equiv 0$ in $\R \setminus (a,a+\eta)$,
\item $\phi$ is affine on  $[a + \eta(1- 2^{-2k + 1}), a + \eta(1 - 2^{-2k})]$ for all $k \in \mathbb{N}$,
\item $\phi(a + \eta(1 - 2^{-2k+2}) + \eta 2^{-2k+1}x) = p_k(x) \varepsilon / (2 k \|p_k\|_{C([0,1])})$ for all $x \in [0,1]$
and all $k \in \mathbb{N}$.
\end{enumerate}
We define $\tilde \sigma := \bar \sigma + \phi$. 
Note that, for this choice of $\tilde \sigma$, 
we clearly have 
$\tilde \sigma \in C(\R)$, 
$\sigma \equiv \tilde \sigma$  in $\R \setminus I$,
and $|\sigma(s) - \tilde \sigma(s)| < \varepsilon$ for all $s \in \R$. 
It remains to show that the neural network $\tilde \psi$ associated with $\tilde \sigma$
satisfies $\closure_{C([0,1])}(\tilde \Psi(D)) = C([0,1])$.
To prove this, we observe that, due to the choice of $p_1$
and the properties of $\bar \sigma$ and $\phi$, we have 
\begin{equation}
\label{eq:randomeq3636}
\frac{2}{\varepsilon} \tilde \sigma\left (a   + \frac12 \eta x\right ) - \frac{2}{\varepsilon}\bar \sigma(a)
= p_1(x)
=
x,\qquad \forall x \in [0,1].
\end{equation}
This equation allows us to turn the functions $\varphi_i^{A_i, b_i}\colon \R \to \R$, $i=1,...,L-1$, into 
identity maps on $[0,1]$ by choosing the weights and biases appropriately and to 
consider w.l.o.g.\ the case $L=1$, cf.\ the proof of \cref{th:nohalfspaces}.
For this one-hidden-layer case, we obtain analogously to \eqref{eq:randomeq3636} that 
\[
\frac{2 k \|p_k\|_{C([0,1])}}{\varepsilon} \tilde \sigma \left (a + \eta(1 - 2^{-2k+2}) + \eta 2^{-2k+1}x\right ) 
-
\frac{2 k \|p_k\|_{C([0,1])}}{\varepsilon}\bar \sigma(a)
=
p_k(x) 
\]
holds for all $x \in [0,1]$ and all $k \in \mathbb{N}$. 
For every $k \in \mathbb{N}$, there thus exists a parameter 
$\alpha_k \in D$ satisfying $\tilde \Psi(\alpha_k) = p_k \in C([0,1])$.
Since $\{p_k\}_{k=1}^\infty$ is dense in $C([0,1])$ by the
Weierstrass approximation theorem, the identity $\closure_{C([0,1])}(\tilde \Psi(D)) = C([0,1])$
now follows immediately. This completes the proof.  
 \end{proof}

Under suitable assumptions on the depth and the widths of $\psi$, 
\cref{lem:dense} can also be extended to the case $d>1$, 
cf.\  \cite[Theorem 4]{Maiorov1999}. 
For some criteria ensuring that the image of $\Psi$ is not dense, 
see \cite[Appendix C3]{Petersen2021}.
We conclude this paper with some additional remarks on 
\cref{th:spuriousminima-nonconstant,th:spuriousminima-constant} and \cref{cor:illposed,cor:LPillposed}:

\begin{remark}~
\begin{itemize}
\item As the proofs of \cref{th:spuriousminima-nonconstant,th:spuriousminima-constant} are constructive, 
they can be used to calculate explicit examples of spurious local minima for training problems of the type \eqref{eq:P}. 
To do so in the situation of 
\cref{th:spuriousminima-nonconstant}, for example, one just has to calculate 
the slope $\bar a \in \R^d$ and the offset $\bar c \in \R$ of the affine linear 
best approximation of the target function $y_T \in C(K)$ w.r.t.\ $\LL$ and then 
plug the resulting values into the formulas in \eqref{eq:baralpha}.
The resulting network parameter 
$\bar \alpha = 
\{ (\bar A_{i}, \bar b_{i})\}_{i=1}^{L+1} \in D$  then 
yields an element of the set $E$ of spurious local minima of  \eqref{eq:P} in \cref{th:spuriousminima-nonconstant}
as desired. We remark that, along these lines, one can also easily 
construct ``bad'' choices of starting values for gradient-based training algorithms 
that cause the training process to terminate with a suboptimal point
(as any reasonable gradient-based algorithm stalls when initialized directly in or near a local minimum). 
We omit including a numerical test of this type here since such experiments have been conducted in various previous works.
We exemplarily mention the numerical investigations in 
\cite[Section 2]{Goldblum2020Truth}, 
in which a deep multilayer perceptron model is trained on CIFAR-10 by means 
of the logistic loss and in which the authors 
provoke gradient-based algorithms to fail by choosing an 
initialization near a spurious minimum of the type discussed above  
(albeit without the knowledge that this local minimum is indeed always spurious);
the numerical experiments in \cite{Safran2018} on the appearance, impact, and 
role of spurious minima in ReLU-networks; 
and the numerical tests in \cite[Sections 3, 4]{Swirszcz2016}, which are concerned with 
training problems for shallow networks 
on the XOR dataset.
\item In contrast to the proofs of \cref{th:spuriousminima-nonconstant,th:spuriousminima-constant},
the proofs of \cref{cor:illposed,cor:LPillposed} are \emph{not} constructive. This significantly 
complicates finding explicit examples of points $u \in  L^p_\mu(K)$ at which the 
function $\pi$ in \cref{cor:LPillposed} is necessarily discontinuous and 
at which the ill-posedness effects documented in \cref{cor:illposed,cor:LPillposed}  become apparent 
-- in particular as these points $u$ can be expected to occupy a comparatively small set, cf.\ \cite{Westphal1989}.
At least to our best knowledge, the construction of explicit data sets and test configurations which provably
illustrate the effects in \cref{cor:illposed,cor:LPillposed} has not been accomplished so far in the literature
(although numerical experiments on instability effects are rather common, cf.\ \cite{Cunningham2000}). 
We remark that, in \cite{Christof2021}, it has been shown that 
 training data vectors, which give rise to ill-posedness effects,
 can be calculated for
finite-dimensional squared loss training problems
by solving a certain $\max$-$\min$-optimization problem, see 
\cite[Lemma 12 and proof of Theorem 15]{Christof2021}. This implicit characterization 
might provide a way for determining ``worst case''-training data sets 
for problems of the type \eqref{eq:P}. We leave this topic for future research.
\end{itemize}
\end{remark}

\bibliographystyle{siamplain}
\bibliography{references}

\begin{thebibliography}{10}

\bibitem{Ainsworth2021}
{\sc M.~Ainsworth and Y.~Shin}, {\em Plateau phenomenon in gradient descent
  training of {RELU} networks: Explanation, quantification, and avoidance},
  SIAM J.~Sci.~Comput., 43 (2021), pp.~3438--3468.

\bibitem{Zhu2019}
{\sc Z.~Allen-Zhu, Y.~Li, and Z.~Song}, {\em A convergence theory for deep
  learning via over-parameterization}, in Proc.~36th Int.~Conf.~Mach.~Learn.,
  K.~Chaudhuri and R.~Salakhutdinov, eds., vol.~97, PMLR, 2019, pp.~242--252.

\bibitem{Arjevani2019SymII}
{\sc Y.~Arjevani and M.~Field}, {\em Analytic study of families of spurious
  minima in two-layer {ReLU} neural networks: A tale of symmetry {II}}, in
  Adv.~Neur.~Inform.~Proc.~Sys., vol.~34, Curran Associates, Inc., 2021.

\bibitem{Auer1996}
{\sc P.~Auer, M.~Herbster, and M.~K. Warmuth}, {\em Exponentially many local
  minima for single neurons}, in Adv.~Neur.~Inform.~Proc.~Sys., D.~S.
  Touretzky, M.~C. Mozer, and M.~E. Hasselmo, eds., vol.~8, Curran Associates,
  Inc., 1996, pp.~316--322.

\bibitem{Benedetto2010}
{\sc J.~J. Benedetto and W.~Czaja}, {\em Integration and Modern Analysis},
  Birkh{\"a}user Advanced Texts, Birkh{\"a}user, Boston, 2010.

\bibitem{Berner2021}
{\sc J.~Berner, P.~Grohs, G.~Kutyniok, and P.~Petersen}, {\em The modern
  mathematics of deep learning}, arxiv:2105.04026v1, 2021.

\bibitem{Blum1992}
{\sc A.~L. Blum and R.~L. Rivest}, {\em Training a 3-node neural network is
  {NP-}complete}, Neur.~Netw., 5 (1992), pp.~117--127.

\bibitem{BonnansShapiro2000}
{\sc J.~F. Bonnans and A.~Shapiro}, {\em Perturbation Analysis of Optimization
  Problems}, Springer Series in Operations Research, Springer, New York, 2000.

\bibitem{Chen2021}
{\sc Z.~Chen, Y.~Cao, D.~Zou, and Q.~Gu}, {\em How much over-parameterization
  is sufficient to learn deep {ReLU} networks?}, arxiv:1911.12360v3, 2020.
\newblock publ.~as conf.~paper, ICLR2021.

\bibitem{Cheridito2021-2}
{\sc P.~Cheridito, A.~Jentzen, and F.~Rossmannek}, {\em Landscape analysis for
  shallow neural networks: complete classification of critical points for
  affine target functions}, arxiv:2103.10922v2, 2021.

\bibitem{Christof2021}
{\sc C.~Christof}, {\em On the stability properties and the optimization
  landscape of training problems with squared loss for neural networks and
  general nonlinear conic approximation schemes}, J.~Mach.~Learn.~Res., 22
  (2021), pp.~1--77.

\bibitem{Christof2021-2}
{\sc C.~Christof and D.~Hafemeyer}, {\em On the nonuniqueness and instability
  of solutions of tracking-type optimal control problems}, Math.~Control
  Relat.~Fields, 12 (2022), pp.~421--431.

\bibitem{Clason2020}
{\sc C.~Clason}, {\em Introduction to Functional Analysis}, Compact Textbooks
  in Mathematics, Birkh{\"a}user, Cham, 2020.

\bibitem{Cohen2021}
{\sc A.~Cohen, R.~DeVore, G.~Petrova, and P.~Wojtaszczyk}, {\em Optimal stable
  nonlinear approximation}, Found.~Comput.~Math.,  (2021).
\newblock published online.

\bibitem{Cooper2020}
{\sc Y.~Cooper}, {\em The critical locus of overparameterized neural networks},
  arxiv:2005.04210v2, 2020.

\bibitem{Cunningham2000}
{\sc P.~Cunningham, J.~Carney, and S.~Jacob}, {\em Stability problems with
  artificial neural networks and the ensemble solution}, Art.~Intell.~Med., 20
  (2000), pp.~217--225.

\bibitem{Cybenko1989}
{\sc G.~Cybenko}, {\em Approximation by superpositions of a sigmoidal
  function}, Math.~Control Signals Systems, 2 (1989), pp.~303--314.

\bibitem{Dauphin2014}
{\sc Y.~Dauphin, R.~Pascanu, C.~G{\"u}lçehre, K.~Cho, S.~Ganguli, and
  Y.~Bengio}, {\em Identifying and attacking the saddle point problem in
  high-dimensional non-convex optimization}, in Adv.~Neur.~Inform.~Proc.~Sys.,
  Z.~Ghahramani, M.~Welling, C.~Cortes, N.~Lawrence, and K.~Q. Weinberger,
  eds., vol.~27, Curran Associates, Inc., 2014, pp.~2933--2941.

\bibitem{DiBenedetto2016}
{\sc E.~DiBenedetto}, {\em Real Analysis}, Birkh{\"a}user Advanced Texts,
  Birkh{\"a}user, second~ed., 2016.

\bibitem{Ding2020}
{\sc T.~Ding, D.~Li, and R.~Sun}, {\em Sub-optimal local minima exist for
  almost all over-parameterized neural networks}, arxiv:1911.01413v3, 2020.

\bibitem{Eftekhari2020}
{\sc A.~Eftekhari}, {\em Training linear neural networks: non-local convergence
  and complexity results}, in Proc.~37th Int.~Conf.~Mach.~Learn., H.~Daumé and
  A.~Singh, eds., vol.~119, PMLR, 2020, pp.~2836--2847.

\bibitem{Ekeland1976}
{\sc I.~Ekeland and R.~Temam}, {\em Convex Analysis and Variational Problems},
  North-Holland Publishing Company, 1976.

\bibitem{Folland1999}
{\sc G.~B. Folland}, {\em Real Analysis: Modern Techniques and Their
  Applications}, Pure and Applied Mathematics: A Wiley Series of Texts,
  Monographs and Tracts, Wiley, second~ed., 1999.

\bibitem{Goldblum2020Truth}
{\sc M.~Goldblum, J.~Geiping, A.~Schwarzschild, M.~Moeller, and T.~Goldstein},
  {\em Truth or backpropaganda? {A}n empirical investigation of deep learning
  theory}, arxiv:1910.00359v3, 2020.
\newblock publ.~as conf.~paper, ICLR2020.

\bibitem{He2020}
{\sc F.~He, B.~Wang, and D.~Tao}, {\em Piecewise linear activations
  substantially shape the loss surfaces of neural networks},
  arxiv:2003.12236v1, 2020.
\newblock publ.~as conf.~paper, ICLR2020.

\bibitem{Hornik1991}
{\sc K.~Hornik}, {\em Approximation capabilities of multilayer feedforward
  networks}, Neur.~Netw., 4 (1991), pp.~251--257.

\bibitem{Kabanikhin2012}
{\sc S.~I. Kabanikhin}, {\em Inverse and Ill-posed Problems: Theory and
  Applications}, De Gruyter, 2012.

\bibitem{Kainen1999}
{\sc P.~C. Kainen, V.~Kůrková, and A.~Vogt}, {\em Approximation by neural
  networks is not continuous}, Neurocomputing, 29 (1999), pp.~47--56.

\bibitem{Kainen2001}
{\sc P.~C. Kainen, V.~Kůrková, and A.~Vogt}, {\em Continuity of approximation
  by neural networks in {$L^p$}-spaces}, Ann.~Oper.~Res., 101 (2001),
  pp.~143--147.

\bibitem{Kawaguchi2016}
{\sc K.~Kawaguchi}, {\em Deep learning without poor local minima}, in
  Adv.~Neur.~Inform.~Proc.~Sys., D.~D. Lee, M.~Sugiyama, U.~V. Luxburg,
  I.~Guyon, and R.~Garnett, eds., vol.~29, Curran Associates, Inc., 2016,
  pp.~586--594.

\bibitem{Laurent2018}
{\sc T.~Laurent and J.~von Brecht}, {\em Deep linear neural networks with
  arbitrary loss: All local minima are global}, in Proc.~35th
  Int.~Conf.~Mach.~Learn., J.~G. Dy and A.~Krause, eds., vol.~80, PMLR, 2018,
  pp.~2908--2913.

\bibitem{DingWidth2022}
{\sc D.~Li, T.~Ding, and R.~Sun}, {\em On the benefit of width for neural
  networks: Disappearance of basins}, SIAM J. Optim., 32 (2022),
  pp.~1728--1758.

\bibitem{Li2018}
{\sc Y.~Li and Y.~Liang}, {\em Learning overparameterized neural networks via
  stochastic gradient descent on structured data}, in Proc.~32nd
  Int.~Conf.~Neur.~Inform.~Proc.~Sys., NIPS'18, Curran Associates Inc., 2018,
  pp.~8168--8177.

\bibitem{Liu2021}
{\sc B.~Liu}, {\em Spurious local minima are common for deep neural networks
  with piecewise linear activations}, arxiv:2102.13233v1, 2021.

\bibitem{Sima2002}
{\sc J.~Šíma}, {\em Training a single sigmoidal neuron is hard}, Neural
  Comput., 14 (2002), pp.~2709--2728.

\bibitem{Maiorov1999}
{\sc V.~Maiorov and A.~Pinkus}, {\em Lower bounds for approximation by {MLP}
  neural networks}, Neurocomputing, 25 (1999), pp.~81--91.

\bibitem{Megginson1998}
{\sc R.~E. Megginson}, {\em An Introduction to {B}anach Space Theory}, no.~183
  in Graduate Texts in Mathematics, Springer, 1998.

\bibitem{Nguyen2018}
{\sc Q.~Nguyen, M.~C. Mukkamala, and M.~Hein}, {\em On the loss landscape of a
  class of deep neural networks with no bad local valleys}, arxiv:1809.10749v2,
  2018.

\bibitem{Nicolae2018}
{\sc A.~Nicolae}, {\em {PLU: T}he piecewise linear unit activation function},
  arxiv:1809.09534, 2018.

\bibitem{Oymak2019}
{\sc S.~Oymak and M.~Soltanolkotabi}, {\em Towards moderate
  overparameterization: global convergence guarantees for training shallow
  neural networks}, IEEE J.~Sel.~Areas Inform.~Theory, 1 (2020), pp.~84--105.

\bibitem{Petersen2021}
{\sc P.~Petersen, M.~Raslan, and F.~Voigtlaender}, {\em Topological properties
  of the set of functions generated by neural networks of fixed size},
  Found.~Comput.~Math, 21 (2021), pp.~375--444.

\bibitem{Petzka2021}
{\sc H.~Petzka and C.~Sminchisescu}, {\em Non-attracting regions of local
  minima in deep and wide neural networks}, J.~Mach.~Learn.~Res., 22 (2021),
  pp.~1--34.

\bibitem{Pinkus1999}
{\sc A.~Pinkus}, {\em Approximation theory of the {MLP} model in neural
  networks}, Acta Numer., 8 (1999), pp.~143--195.

\bibitem{Safran2018}
{\sc I.~Safran and O.~Shamir}, {\em Spurious local minima are common in
  two-layer {ReLU} neural networks}, in Proc.~35th Int.~Conf.~Mach.~Learn.,
  J.~G. Dy and A.~Krause, eds., vol.~80, PMLR, 2018, pp.~4430--4438.

\bibitem{Safran2021}
{\sc I.~Safran, G.~Yehudai, and O.~Shamir}, {\em The effects of mild
  over-parameterization on the optimization landscape of shallow {ReLU} neural
  networks}, in Proc.~34th~Conf.~Learn.~Theory, M.~Belkin and S.~Kpotufe, eds.,
  vol.~134, PMLR, 2021, pp.~3889--3934.

\bibitem{Saxe2013}
{\sc A.~M. Saxe, J.~L. McClelland, and S.~Ganguli}, {\em Exact solutions to the
  nonlinear dynamics of learning in deep linear neural networks},
  arxiv:1312.6120v3, 2014.

\bibitem{Sun2019}
{\sc R.~Sun}, {\em Optimization for deep learning: theory and algorithms},
  arxiv:1912.08957v1, 2019.

\bibitem{Sun2020}
{\sc R.~Sun, D.~Li, S.~Liang, T.~Ding, and R.~Srikant}, {\em The global
  landscape of neural networks: an overview}, IEEE Signal Process.~Mag., 37
  (2020), pp.~95--108.

\bibitem{Swirszcz2016}
{\sc G.~Swirszcz, W.~M. Czarnecki, and R.~Pascanu}, {\em Local minima in
  training of neural networks}, arxiv:1611.06310v2, 2016.

\bibitem{Venturi2019}
{\sc L.~Venturi, A.~S. Bandeira, and J.~Bruna}, {\em Spurious valleys in
  one-hidden-layer neural network optimization landscapes},
  J.~Mach.~Learn.~Res., 20 (2019), pp.~1--34.

\bibitem{Vlasov1970}
{\sc L.~P. Vlasov}, {\em Almost convex and {Chebyshev} sets}, Math.~Notes
  Acad.~Sci.~USSR, 8 (1970), pp.~776--779.

\bibitem{Westphal1989}
{\sc U.~Westphal and J.~Frerking}, {\em On a property of metric projections
  onto closed subsets of {H}ilbert spaces}, Proc.~Amer.~Math.~Soc., 105 (1989),
  pp.~644--651.

\bibitem{Yoshida2019}
{\sc Y.~Yoshida and M.~Okada}, {\em Data-dependence of plateau phenomenon in
  learning with neural network---statistical mechanical analysis},
  Adv.~Neur.~Inform.~Proc.~Sys., 32 (2019).

\bibitem{Yu1995}
{\sc X.-H. Yu and G.-A. Chen}, {\em On the local minima free condition of
  backpropagation learning}, IEEE Trans.~Neural Netw., 6 (1995),
  pp.~1300--1303.

\bibitem{Yun2018}
{\sc C.~Yun, S.~Sra, and A.~Jadbabaie}, {\em Small nonlinearities in activation
  functions create bad local minima in neural networks}, arxiv:1802.03487v4,
  2019.
\newblock publ.~as conf.~paper, ICLR2019.

\bibitem{Zou2020}
{\sc D.~Zou, P.~M. Long, and Q.~Gu}, {\em On the global convergence of training
  deep linear {ResNets}}, arxiv:2003.01094v1, 2020.
\newblock publ.~as conf.~paper, ICLR2020.

\end{thebibliography}

\end{document}